%% file: neurips_2021.tex
\documentclass{article}

    \PassOptionsToPackage{numbers}{natbib}

\usepackage[final]{neurips_2021}

\usepackage[utf8]{inputenc} 
\usepackage[T1]{fontenc}    
\usepackage{hyperref}       
\usepackage{url}            
\usepackage{booktabs}       
\usepackage{amsfonts}       
\usepackage{amsmath}
\usepackage{amsthm}
\usepackage{bm}
\usepackage{nicefrac}       
\usepackage{microtype}      
\usepackage{xcolor,colortbl}         
\usepackage{subfigure}
\usepackage[inline]{enumitem}
\usepackage{multirow}
\usepackage{graphicx}
\usepackage{textcomp}
\usepackage{mathtools}
\usepackage{caption}
\usepackage{wrapfig}
\usepackage{ifthen}
\usepackage{pgfplots}
\input{insbox}
\usepackage{algorithm}
\usepackage{algpseudocode}%
\usepackage{tikz}
\usepackage{siunitx}
\usepgfplotslibrary{groupplots}
\usetikzlibrary{shapes,fit,backgrounds,arrows,decorations.markings,arrows.meta}
\usetikzlibrary{matrix, calc, positioning}
\usepackage{fixltx2e}
\usepackage{pgfplotstable}
\usepackage{xstring}

\newtheorem{theorem}{Theorem}[section]
\newtheorem{corollary}{Corollary}[theorem]

\newtheorem{proposition}[theorem]{Proposition}
\newtheorem{definition}[theorem]{Definition}

\newcommand{\flickr}{{\fontfamily{lmtt}\selectfont Flickr}}
\newcommand{\reddit}{{\fontfamily{lmtt}\selectfont Reddit}}
\newcommand{\yelp}{{\fontfamily{lmtt}\selectfont Yelp}}
\newcommand{\arxiv}{{\fontfamily{lmtt}\selectfont ogbn-arxiv}}
\newcommand{\products}{{\fontfamily{lmtt}\selectfont ogbn-products}}
\newcommand{\papers}{{\fontfamily{lmtt}\selectfont ogbn-papers100M}}
\newcommand{\arxivshort}{{\fontfamily{lmtt}\selectfont arxiv}}
\newcommand{\productsshort}{{\fontfamily{lmtt}\selectfont products}}
\newcommand{\papersshort}{{\fontfamily{lmtt}\selectfont papers100M}}
\newcommand{\collab}{{\fontfamily{lmtt}\selectfont ogbl-collab}}
\newcommand{\collabshort}{{\fontfamily{lmtt}\selectfont collab}}

\newcommand{\shadow}{\textsc{shaDow-GNN}}
\newcommand{\shadowsage}{\textsc{shaDow-SAGE}}

\newcommand{\shadowgat}{\textsc{shaDow-GAT}}
\newcommand{\shadowgcn}{\textsc{shaDow-GCN}}
\newcommand{\shadowgin}{\textsc{shaDow-GIN}}
\newcommand{\shadowsgc}{\textsc{shaDow-SGC}}

\newcommand{\std}[1]{\footnotesize{\color{gray}$\pm$#1}}

\newcommand{\ie}{\emph{i.e.}}
\newcommand{\eg}{\emph{e.g.}}
\newcommand{\parag}[1]{\noindent\textbf{#1}\hspace{.3cm}}

\input{math_commands}

\definecolor{airforceblue}{rgb}{0.36, 0.54, 0.66}
\definecolor{colR1}{rgb}{0.0, 0.18, 0.39}
\definecolor{LightCyan}{HTML}{d0FFFF}
\definecolor{LightGreen}{HTML}{ECFFB9}
\definecolor{HeavyCyan}{HTML}{80ffff}
\definecolor{LightRed}{HTML}{F1EAFE}
\definecolor{LightYellow}{HTML}{ffffd0}
\definecolor{LightGrey}{HTML}{F9E8E2}
\newcolumntype{a}{>{\columncolor{LightCyan}}c}  
\newcolumntype{i}{>{\columncolor{white}}c}      

\pgfplotsset{
    line and fill/.style={
        legend image code/.code={%
          \draw [##1,fill=none, thick] (0mm,0mm) -- (4mm,0mm);
        },
    },
}
\newcommand{\errorband}[6][]{
    \addplot [draw=none, stack plots=y, forget plot] 
      table [x={#3},y expr=\thisrow{#4}-\thisrow{#5}] {#2};
    \addplot [fill=gray!40, stack plots=y, opacity=0.5, #1, draw=none, forget plot] 
      table [x={#3},y expr=2*\thisrow{#5}] {#2} \closedcycle;
    \addplot [stack plots=y, draw=none, forget plot] 
      table [x={#3},y expr=-(\thisrow{#4}+\thisrow{#5})] {#2};
    \addplot [forget plot, thick, #1, fill=none] 
      table [x={#3},y expr=\thisrow{#4}] {#2};
    \addlegendimage{line and fill,#1}
    \addlegendentry{#6}
}
\pgfplotsset{compat=1.16}

\newcommand\blfootnote[1]{%
  \begingroup
  \renewcommand\thefootnote{}\footnote{#1}%
  \addtocounter{footnote}{-1}%
  \endgroup
}
\input{data/deep_gcn}
\input{data/data_gpu_mem_speed}

\title{Decoupling the Depth and Scope of\\Graph Neural Networks}

\author{
    Hanqing Zeng\\
    USC\\
    \texttt{zengh@usc.edu}
    \And
    Muhan Zhang\\
    Peking University, BIGAI\\
    \texttt{muhan@pku.edu.cn}
    \And
    Yinglong Xia\\
    Facebook AI\\
    \texttt{yxia@fb.com}
    \And
    Ajitesh Srivastava\\
    USC\\
    \texttt{ajiteshs@usc.edu}
    \And
    Andrey Malevich\\
    Facebook AI\\
    \texttt{amalevich@fb.com}
    \And
    Rajgopal Kannan\\
    US ARL\\
    \texttt{rajgopal.kannan.civ@mail.mil}
    \And
    Viktor Prasanna\\
    USC\\
    \texttt{prasanna@usc.edu}
    \And
    Long Jin\\
    Facebook AI\\
    \texttt{longjin@fb.com}
    \And
    Ren Chen\\
    Facebook AI\\
    \texttt{renchen@fb.com}
}

\begin{document}

\maketitle

\begin{abstract}
State-of-the-art Graph Neural Networks (GNNs) have limited scalability with respect to the graph and model sizes.
On large graphs, increasing the model depth often means exponential expansion of the scope (\ie, receptive field). 
Beyond just a few layers, two fundamental challenges emerge: 
\begin{enumerate*}
  \item degraded \emph{expressivity} due to oversmoothing, and 
  \item expensive \emph{computation} due to neighborhood explosion. 
\end{enumerate*}
We propose a design principle to decouple the depth and scope of GNNs --
to generate representation of a target entity (\ie, a node or an edge), we first extract a localized subgraph as the \emph{bounded-size} scope, and then apply a GNN of arbitrary depth on top of the subgraph. 
A properly extracted subgraph consists of a small number of critical neighbors, while excluding irrelevant ones. 
The GNN, no matter how deep it is, smooths the local neighborhood into informative representation rather than oversmoothing the global graph  into ``white noise''. 
Theoretically, decoupling improves the GNN expressive power from the perspectives of graph signal processing (GCN), function approximation (GraphSAGE) and topological learning (GIN). 
Empirically, on seven graphs (with up to 110M nodes) and six backbone GNN architectures, our design achieves significant accuracy improvement with orders of magnitude reduction in computation and hardware cost. 
\blfootnote{Correspondence to: Muhan Zhang, \texttt{muhan@pku.edu.cn}}

\end{abstract}

\input{1_intro}
\input{2_preliminary}
\input{3_method}
\input{4_related_work}
\input{5_exp}
\input{6_conclusion}

\newpage

\newpage
\appendix

\input{7_appendix}

\end{document}

%% file: math_commands.tex
\usepackage{amsmath,amsfonts,bm}

\newcommand{\trans}{\mathsf{T}}
\newcommand{\ee}{\bar{\bar{\bm{e}}}}

\newcommand{\sample}{\text{{\fontfamily{lmtt}\selectfont EXTRACT}}}
\newcommand{\relu}{\text{{\fontfamily{lmtt}\selectfont ReLU}}}
\newcommand{\func}[2][]{
    {\text{\fontfamily{lmtt}\selectfont #1}\left(#2\right)}}
\newcommand{\degv}[2][]{
    \delta_{[#1]}\left(#2\right)}
\newcommand{\G}[1][]{
    \IfEq{#1}{}
    {\mathcal{G}}
    {\mathcal{G}}_{#1}}
\newcommand{\V}[1][]{
    \IfEq{#1}{}
    {\mathcal{V}}
    {\mathcal{V}_{#1}}}
\newcommand{\E}[1][]{
    \IfEq{#1}{}
    {\mathcal{E}}
    {\mathcal{E}_{#1}}}
\newcommand{\N}[1][]{
    \IfEq{#1}{}
    {\mathcal{N}}
    {\mathcal{N}}_{#1}}
\newcommand{\X}[1][]{
    \IfEq{#1}{}
    {\bm{X}}
    {\bm{X}^{\paren{#1}}}}
\newcommand{\Asym}[1][]{
    \IfEq{#1}{}
    {\widetilde{\bm{A}}}
    {\widetilde{\bm{A}}_{#1}}}
\newcommand{\Arw}[1][]{
    \IfEq{#1}{}
    {\widehat{\bm{A}}}
    {\widehat{\bm{A}}_{#1}}}
\newcommand{\deltarw}{\widehat{\bm{\delta}}}
\newcommand{\A}[1][]{
    \IfEq{#1}{}
    {\bm{A}}
    {\bm{A}}_{#1}}
\newcommand{\Hx}[1][]{
    \IfEq{#1}{}
    {\bm{H}}
    {\bm{H}}^{(#1)}}

\newcommand{\W}[1][]{
    \IfEq{#1}{}
    {\bm{W}}
    {\bm{W}^{\paren{#1}}}}

\newcommand{\size}[1]{\left\lvert #1 \right\rvert}
\newcommand{\paren}[1]{\left( #1 \right)}
\DeclarePairedDelimiterX\set[1]\lbrace\rbrace{\def\given{\;\delimsize\vert\;}#1}

\def\eqref#1{equation~\ref{#1}}

\def\ceil#1{\lceil #1 \rceil}

\def\1{\bm{1}}

\def\vh{{\bm{h}}}

\def\mA{{\bm{A}}}

\DeclareMathAlphabet{\mathsfit}{\encodingdefault}{\sfdefault}{m}{sl}
\SetMathAlphabet{\mathsfit}{bold}{\encodingdefault}{\sfdefault}{bx}{n}

\newcommand{\R}{\mathbb{R}}

%% file: data/deep_gcn.tex
\pgfplotstableread[col sep=space]{
1	0.4776	0	0.4592	0	0.9423	0	0.9454	0	0.658	0	0.6593	0
3	0.5193	0	0.506	0	0.9536	0	0.9485	0	0.6984	0	0.7034	0
5	0.5238	0	0.5197	0	0.9497	0	0.945	0	0.697	0	0.7014	0
10	0.523	0	0.53	0	0.9455	0	0.9348	0	0.6932	0	0.6908	0
20	0.5224	0	0.5021	0	0.9427	0	0.9087	0	0.6896	0	0.6728	0
30	0.5236	0	0.463	0	0.9442	0	0.8782	0	0.6913	0	0.6427	0
40	0.5216	0	0.4455	0	0.941	0	0.8525	0	0.6857	0	0.6266	0
}\datatable

%% file: data/data_gpu_mem_speed.tex
\pgfplotstableread[col sep=space]{
1.7 1 0 1 0
2.8 2 0 2 0
}\dataMemSpeed

%% file: 1_intro.tex
\section{Introduction}
\label{sec: intro}

Graph Neural Networks (GNNs) have now become the state-of-the-art models for graph mining \citep{gnn_survey, gnn_survey_snap, gnn_social_survey}, facilitating applications such as social recommendation \citep{mc_nips17, pinsage, pinnersage}, knowledge understanding \citep{rgcn, kg_kdd19, kg_iclr20} and drug discovery \citep{gnn_drug, gnn_drug2}. 
With the numerous architectures proposed \citep{gcn_welling, graphsage, gat, gin}, it still remains an open question how to effectively scale up GNNs with respect to both the model size and graph size. 
There are two fundamental obstacles when we increase the number of GNN layers: 

\begin{itemize}[leftmargin=0.5cm]
    \setlength\itemsep{0.1em}
    \item \emph{Expressivity} challenge (\ie, oversmoothing \citep{aaai_oversmooth, Oono2020, dropedge, dropedge_journal}): \hspace{.2cm}iterative mixing of neighbor features collapses embedding vectors of different nodes into a fixed, low-dimensional subspace. 
    \item \emph{Scalability} challenge (\ie, neighbor explosion \citep{vrgcn, fastgcn, clustergcn, graphsaint}): \hspace{.2cm}
    recursive expansion of multi-hop neighborhood results in exponentially growing receptive field size (and thus computation cost). 
\end{itemize}

To address the expressivity challenge, most remedies focus on neural architecture exploration: 
\cite{gat, graphsage, gin, deepergcn} propose more expressive aggregation functions when propagating neighbor features. \cite{jknet, deepgcn, asgcn, gcnii, mixhop, break_the_ceiling, kdd_deep} use residue-style design components to construct flexible and dynamic receptive fields. Among them, \cite{jknet, deepgcn, asgcn} use skip-connection across multiple GNN layers, and \cite{gcnii, mixhop, break_the_ceiling, kdd_deep} encourage multi-hop message passing within each single layer. 
As for the scalability challenge, sampling methods have been explored to improve the training speed and efficiency. 
Importance based layer-wise sampling \cite{fastgcn, vrgcn, ladies} and subgraph-based sampling \cite{graphsaint_parallel, clustergcn, graphsaint} alleviate neighbor explosion, while preserving training accuracy. Unfortunately, such sampling methods cannot be naturally generalized to inference without accuracy loss (see also Section \ref{sec: related}). 

The above lines of research have only guided us to partial solutions. Yet what is the root cause of both the expressivity and scalability challenges?
Setting aside the design of GNN architectures or sampling schemes, we provide an alternative perspective by interpretting the \emph{data} in a different way. 

\parag{Two views on the graph.}
Given an input graph $\G$ with node set $\V$, 
the most straightforward way to understand $\G$ is by viewing it as a single \emph{global} graph. So any two nodes $u$ and $v$ belong to the same $\G$, and if $u$ and $v$ lie in the same connected component, they will ultimately see each other in their own neighborhood no matter how far away $u$ and $v$ are. 
Alternative to the above \emph{global} view, we can take a \emph{local} view on $\G$. For each node $v$, there is a \emph{latent} $\G[{[v]}]$ surrounding it which captures the characteristics of just $v$ itself.  
The full $\G$ is \emph{observed} (by the data collection process) as the union of all such $\G[{[v]}]$. 
Consequently, $\V[{[v]}]$ rather than $\V$ defines $v$'s neighborhood: if $u\not\in\V[{[v]}]$, $v$ will never consider $u$ as a neighbor. 
Our ``decoupling'' design is based on the local view. 

\parag{Scope of GNNs.}
Both the expressivity and scalability challenges are closely related to the enlargement of the GNN's scope (\ie, receptive field). 
More importantly, how we define the scope is determined by how we view $\G$. 
With the global view above, an $L$-layer GNN has the scope of the full $L$-hop neighborhood. 
With the local view, the GNN scope is simply $\V[{[v]}]$ regardless of the GNN depth. 
The two existing lines of research, one on architectural exploration and the other on sampling, both take the \emph{global} view.  
Consequently, the depth (\ie, number of layers) and scope of such GNNs are \emph{tightly coupled}.
Such coupling significantly limits the design space exploration of GNNs with various depths \cite{design_space}. 
Consider the example of {\products}, a medium-scale graph in Open Graph Benchmark \cite{ogb}. The average number of 4-hop neighbors is around 0.6M, corresponding to $25\%$ of the full graph size. 
To generate representation of a single target node, a 4-layer coupled GNN needs to propagate features from the 0.6M neighbors. 
Such propagation can be inefficient or even harmful since most nodes in the huge neighborhood would be barely relevant to the target node.

\parag{Decoupling the GNN depth and scope.}
Taking the local view on $\G$, we propose a general design principle to decouple the GNN depth and scope. 
To generate the representation of the target node $v$, we first extract from $\G$ a small subgraph $\G[{[v]}]$ surrounding $v$.
On top of $\G[{[v]}]$, we apply a GNN whose number of layers and message passing functions can be flexibly chosen. 
``Decoupling'' means we treat the scope extraction function and GNN depth as two independently tuned parameters -- effectively we introduce a new dimension in the GNN design space. 
We intuitively illustrate the benefits of decoupling by an example GNN construction, where the scope is the $L$-hop neighborhood and depth is $L'$ ($L'>L$). 
When we use more layers ($L'$) than hops ($L$), each pair of subgraph nodes may exchange messages multiple times. 
The extra message passing helps the GNN better absorb and embed the information within scope, and thus leads to higher expressivity. 
We further justify the above intuition with multifaceted theoretical analysis. 
From the graph signal processing perspective, we prove that decoupled-GCN performs local-smoothing rather than oversmoothing, as long as the scopes of different target nodes are different. 
From the function approximation perspective, we construct a linear target function on neighbor features and show that decoupling the GraphSAGE model reduces the function approximation error. 
From the topological learning perspective, we apply deep GIN-style message passing to differentiate non-regular subgraphs of a regular graph. As a result, our model is more powerful than the 1-dimensional Weisfeiler-Lehman test \cite{wl}.

\parag{Practical implementation: {\shadow}.}
The decoupling principle leads to a practical implementation, {\shadow}:  \underline{D}ecoupled GNN on a \underline{sha}ll\underline{ow} subgraph. 
In {\shadow}, the scope is a shallow yet informative subgraph, only containing a fraction of the 2- or 3-hop neighbors of $\G$ (see Section \ref{sec: exp}). On the other hand, the model of {\shadow} is deeper (\eg, $L'=5$). 
To efficiently construct the shallow scope on commodity hardware, we propose various subgraph extraction functions. 
To better utilize the subgraph node embeddings after deep message passing, we propose neural architecture extensions such as pooling and ensemble. 
Empirically, our ``decoupling'' design improves both the accuracy and scalability.
On seven benchmarks (including the largest {\papers} graph with 111M nodes) and across two graph learning tasks, {\shadow}s achieve significant accuracy gains compared to the original models. 
Meanwhile, the computation and hardware costs are reduced by orders of magnitude. 
Our code is available at \url{https://github.com/facebookresearch/shaDow_GNN}

%% file: 2_preliminary.tex
\section{Preliminaries}
\label{sec: preliminary}
Let $\G\paren{\V, \E, \X}$ be an undirected graph, with node set $\V$, edge set $\E\subseteq\V\times \V$ and node feature matrix $\X\in\mathbb{R}^{|\V|\times d}$. 
Let $\mathcal{N}_v$ denote the set of $v$'s direct neighbors in $\G$. 
The $u$\textsuperscript{th} row of $\X$ corresponds to the length-$d$ feature of node $u$. 
Let $\bm{A}$ be the adjacency matrix of $\G$ where $A_{u,v}=1$ if edge $(u, v)\in\E$ and $A_{u,v}=0$ otherwise. 
Let $\bm{D}$ be the diagonal degree matrix of $\bm{A}$. 
Denote $\Asym=\bm{D}_*^{-\frac{1}{2}}\bm{A}_*\bm{D}_*^{-\frac{1}{2}}$ as the adjacency matrix after symmetric normalization (``$*$'' means augmented with self-edges), and $\Arw=\bm{D}^{-1}\bm{A}$ (or $\bm{D}_*^{-1}\bm{A}_*$) as the one after random walk normalization.  Let subscript ``$[u]$'' mark the quantities corresponding to a subgraph surrounding node $u$. 
For example, the subgraph itself is $\G[{[u]}]$. 
Subgraph matrices $\X_{[v]}$ and $\bm{A}_{[v]}$ have the same dimension as the original $\X$ and $\A$. 
Yet, row vector $\left[\X_{[v]}\right]_u=\bm{0}$ for $u\not\in\V[{[v]}]$. 
Element $\left[\bm{A}_{[v]}\right]_{u,w}=0$ if either $u\not\in\V[{[v]}]$ or $w\not\in\V[{[v]}]$. 
For an $L$-layer GNN, let superscript ``$(\ell)$'' denote the layer-$\ell$ quantities. Let $d^{(\ell)}$ be the number of channels for layer $\ell$; $\Hx[\ell-1]\in\mathbb{R}^{|\V|\times d^{(\ell-1)}}$ and $\Hx[\ell]\in\mathbb{R}^{|\V|\times d^{(\ell)}}$ be the input and output features.
So $\Hx[0] = \X$ and $d^{(0)}=d$. 
Further, let $\sigma$ be the activation and $\bm{W}^{\paren{\ell}}$ be the learnable weight. 
For example, a GCN layer performs $\Hx[\ell]=\sigma\paren{\Asym \Hx[\ell-1]\bm{W}^{\paren{\ell}}}$. 
A GraphSAGE layer performs $\Hx[\ell] = \sigma\paren{\Hx[\ell-1]\bm{W}_1^{\paren{\ell}} + \Arw\Hx[\ell-1]\bm{W}_2^{\paren{\ell}}}$. 

Our analysis in Section \ref{sec: method} mostly focuses on the node classification task. Yet our design can be generalized to the link prediction task, as demonstrated by our experiments in Section \ref{sec: exp}. 

\begin{definition}
(Depth of subgraph) Assume the subgraph $\G[{[v]}]$ is connected. The depth of $\G[{[v]}]$ is defined as $\max_{u\in\V[{[v]}]}d\paren{u,v}$, where $d\paren{u,v}$ denotes the shortest path distance from $u$ to $v$. 
\end{definition}

The above definition enables us to make comparison such as ``the GNN is deeper than the subgraph''. For ``decoupling the depth and scope'', we refer to the model depth rather than the subgraph depth. 

%% file: 3_method.tex
\section{Decoupling the Depth and Scope of GNNs}
\label{sec: method}

``Decoupling the depth and scope of GNNs'' is a design principle to improve the expressivity and scalability of GNNs without modifying the layer architecture. 
We name a GNN after decoupling a {\shadow} (see Section \ref{sec: shadow practical} for explanation of the name). 
Compared with a normal GNN, {\shadow} contains an additional component: the subgraph extractor $\sample$. 
To generate embedding of a target node $v$, {\shadow} proceeds as follows:
\begin{enumerate*}
\item We use $\sample\paren{v, \G}$ to return a connected $\G[{[v]}]$, where $\G[{[v]}]$ is a subgraph containing $v$, and the depth of $\G[{[v]}]$ is $L$. 
\item We build an $L'$-layer GNN on $\G[{[v]}]$ by treating $\G[{[v]}]$ as the \emph{new} full graph and by ignoring all nodes / edges not in $\G[{[v]}]$. 
\end{enumerate*}
So $\G[{[v]}]$ is the scope of {\shadow}. 
The key point reflecting ``decoupling'' is that $L'> L$. 

A normal GNN is closely related to a {\shadow}. 
Under the normal setup, an $L$-layer GNN operates on the full $\G$ and propagates the influence from all the neighbors up to $L$ hops away from $v$. 
Such a GNN is equivalent to a model where {\sample} returns the full $L$-hop subgraph and $L'=L$. 

We theoretical demonstrate how {\shadow} improves expressivity from three different angles. 
On {\shadowgcn} (Section \ref{sec: analysis gcn}), we come from the \emph{graph signal processing} perspective. 
The GCN propagation can be interpreted as applying filtering on the node signals \cite{sgc}. Deep models correspond to high-pass filters. 
Filtering the local graph $\G[{[v]}]$ preserves richer information than the global $\G$. 
On {\shadowsage} (Section \ref{sec: shadow-sage}), we view the GNN as a \emph{function approximator}. 
We construct a target function and study how decoupling reduces the approximation error. 
On {\shadowgin} (Section \ref{sec: shadow-gin}), we focus on learning \emph{topological information}. 
We show that decoupling helps capture local graph structure which the 1D Weisfeiler-Lehman test fails to capture. 

\subsection{Expressivity Analysis on {\shadowgcn}: Graph Signal Processing Perspective}
\label{sec: analysis gcn}

GCNs \cite{gcn_welling} suffer from ``oversmoothing'' \cite{aaai_oversmooth} -- 
Each GCN layer smooths the features of the direct (\ie, 1-hop) neighbors, and many GCN layers smooths the features of the full graph. 
Eventually, such repeated smoothing process propagates to any target node just the averaged feature of all $\V$. 
``Oversmoothing'' thus incurs significant information loss by wiping out all local information. 

Formally, suppose the original features $\X$ reside in a high-dimensional space $\mathbb{R}^{\size{\V}\times d}$. Oversmoothing pushes $\X$ towards a low-dimensional subspace $\mathbb{R}^{\size{\V}\times d'}$, where $d' < d$. 
Corresponding analysis comes from two perspectives: oversmoothing by \emph{a deep GCN}, and oversmoothing by \emph{repeated GCN-style propagation}. 
The former considers the full neural network with non-linear activation, weight and bias. 
The later characterizes the aggregation matrix 
$\bm{M}=\lim_{L\rightarrow \infty}\Asym^L \X$. 
It is shown that even with the vanilla architecture, a deep GCN with bias parameters does \emph{not} oversmooth \cite{dropedge_journal}. 
In addition, various tricks \cite{pairnorm, dropedge, gcnii} can prevent oversmoothing from the neural network perspective. 
However, a deep GCN still suffers from accuracy drop, indicating that the GCN-style propagation (rather than other GCN components like activation and bias) may be the fundamental reason causing difficulty in learning. 
Therefore, we study the asymptotic behavior of the aggregation matrix $\bm{M}$ under the normal and \textsc{shaDow} design. 
In other words, here in Section \ref{sec: analysis gcn}, we ignore the non-linear activation and bias parameters. 
Such setup is consistent with many existing literature such as \cite{aaai_oversmooth, break_the_ceiling, gcnii, pairnorm}. 

\begin{proposition}
\label{prop: shadow embedding}
$\infty$ number of feature propagation by {\shadowgcn} leads to
\begin{equation}
    \bm{m}_{[v]} = \left[e_{[v]}\right]_v\cdot \paren{\bm{e}_{[v]}^\trans\X_{[v]}}
\end{equation} 
where $\bm{e}_{[v]}$ is defined by $\left[e_{[v]}\right]_u=\sqrt{\frac{\degv[v]{u}}{\sum_{w\in\V[{[v]}]}\degv[v]{w}}}$; $\degv[v]{u}$ returns the degree of $u$ in $\G[{[v]}]$ plus 1. 
\end{proposition}

\parag{Oversmoothing by normal GCN propagation.}
With a large enough $L$, the full $L$-hop neighborhood becomes $\V$ (assuming connected $\G$). So $\forall\; u, v$, we have $\G[{[u]}] = \G[{[v]}]=\G$, implying $\bm{e}_{[u]}=\bm{e}_{[v]}$ and $\X_{[u]}=\X_{[v]}=\X$. 
From Proposition \ref{prop: shadow embedding},
the aggregation converges to a point where \emph{no} feature and \emph{little} structural information of the target is preserved. The only information in $\bm{m}_{[v]}$ is $v$'s degree. 

\parag{Local-smoothing by {\shadowgcn} propagation.} With a fixed subgraph, 
no matter how many times we aggregate using $\Asym_{[v]}$, the layers will not include the faraway irrelevant nodes. 
From Proposition \ref{prop: shadow embedding}, $\bm{m}_{[v]}$ is a linear combination of the neighbor features $\X_{[v]}$. 
Increasing the number of layers only pushes the \emph{coefficients} of each neighbor features to the stationary values. The \emph{domain} $\X_{[v]}$ of such linear transformation is solely determined by {\sample} and is independent of the model depth. 
Intuitively, if {\sample} picks non-identical subgraphs for two nodes $u$ and $v$, the aggregations should be different due to the different domains of the linear transformation. 
Therefore, {\shadowgcn} preserves \emph{local} feature information whereas normal GCN preserves \emph{none}. 
For structural information in $\bm{m}_{[v]}$, note that $\bm{e}_{[v]}$ is a normalized degree distribution of the subgraph around $v$, and $\left[e_{[v]}\right]_v$ indicates the role of the target node in the subgraph. 
By simply letting {\sample} return the $1$-hop subgraph,  $\left[e_{[v]}\right]_v$ alone already contains all the information preserved by a normal GCN, which is $v$'s degree in $\G$. 
For the general {\sample}, $\bm{e}_{[v]}$ additionally reflects $v$'s ego-net structure.
Thus, a deep {\shadowgcn} preserves more structural information than a deep GCN. 

\begin{theorem}
\label{thm: non-oversmoothing}
Let $\overline{\bm{m}}_{[v]}=\phi_{\G}\paren{v}\cdot\bm{m}_{[v]}$ 
where $\phi_{\G}$ is any non-zero function only depending on the structural property of $v$. 
Let $\mathcal{M}=\set{\overline{\bm{m}}_{[v]}\given v\in\V}$. 
Given $\G$, {\sample} and some continuous probability distribution in $\mathbb{R}^{\size{\V}\times d}$ to generate $\X$, then $\overline{\bm{m}}_{[v]}\neq \overline{\bm{m}}_{[u]}$ if $\V[{[u]}]\neq \V[{[v]}]$, almost surely. 
\end{theorem}

\begin{corollary}
\label{coro: size-n sampler}
Consider ${\sample}_1$, where $\forall v\in\V,\;\size{\V[{[v]}]}\leq n$. Then $\size{\mathcal{M}}\geq \left\lceil{\frac{\size{\V}}{n}}\right\rceil$ a.s. 
\end{corollary}

\begin{corollary}
\label{coro: non-identical sampler}
Consider ${\sample}_2$, where $\forall\;u,v \in\V, \;\V[{[v]}]\neq\V[{[u]}]$. Then $\size{\mathcal{M}}=\size{\V}$ a.s. 
\end{corollary}

Theorem \ref{thm: non-oversmoothing} proves {\shadowgcn} does not oversmooth: 
\begin{enumerate*}
\item A normal GCN pushes the aggregation of same-degree nodes to the same point, while {\shadowgcn} with ${\sample}_2$ ensures any two nodes (even with the same degree) have different aggregation. 
\item A normal GCN wipes out all information in $\X$ after many times of aggregation, while
{\shadowgcn} always preserves feature information. 
\end{enumerate*}
Particularly, with $\phi_{\G}\paren{v}=\paren{\delta_{[v]}(v)}^{-1/2}$, a normal GCN generates only one unique value of $\overline{\bm{m}}$ for all $v$. 
By contrast, {\shadow} generates $\size{\V}$ different values for any $\phi_{\G}$ function. 

\subsection{Expressivity Analysis on {\shadowsage}: Function Approximation Perspective}
\label{sec: shadow-sage}

We compare the expressivity by showing
\begin{enumerate*}
\item {\shadowsage} can express all functions GraphSAGE can, and 
\item {\shadowsage} can express some function GraphSAGE cannot. 
\end{enumerate*}
Recall, a GraphSAGE layer performs the following: $\bm{h}_v^{\paren{\ell}} = \sigma\paren{\paren{\bm{W}_1^{\paren{\ell}}}^\trans\bm{h}_v^{\paren{\ell-1}}+\paren{\bm{W}_2^{\paren{\ell}}}^\trans\paren{\frac{1}{\size{\mathcal{N}_v}}\sum_{u\in\mathcal{N}_v}\bm{h}_u^{\paren{\ell-1}}}}$. 
We can prove Point 1 by making an $L'$-layer {\shadowsage} identical to an $L$-layer GraphSAGE with the following steps: 
\begin{enumerate*}
\item let {\sample} return the full $L$-hop neighborhood, and 
\item set $\bm{W}_1^{\paren{\ell}}=\bm{I}$, $\bm{W}_2^{\paren{\ell}}=\bm{0}$ for $L+1\leq\ell\leq L'$. 
\end{enumerate*}
For point 2, we consider a target function:
$\tau\paren{\X, \G[{[v]}]} = C\cdot \sum_{u\in\V[{[v]}]}\delta_{[v]}\paren{u}\cdot\bm{x}_u$
for some neighborhood $\G[{[v]}]$, scaling constant $C$ and $\delta_{[v]}\paren{u}$ as defined in Proposition \ref{prop: shadow embedding}. 
An expressive model should be able to learn well this simple linear function $\tau$. 

GraphSAGE cannot learn $\tau$ accurately, while {\shadowsage} can. 
We first show the GraphSAGE case. 
Let the depth of $\G[{[v]}]$ be $L$. 
Firstly, we need GraphSAGE to perform message passing for exactly $L$ times (where such a model can be implemented by, \eg, $L$ layers or $L'$ layers with $\bm{W}_2=\bm{0}$ for $L'-L$ layers). 
Otherwise, the extra $L'-L$ message passings will propagate influence from nodes $v'\not\in\V[{[v]}]$, violating the condition that $\tau$ is independent of $v'$. 
Next, suppose GraphSAGE can learn a function $\zeta$ such that on some $\G[{[v]}]'$, we have $\zeta\paren{\G[{[v]}]'}=\tau\paren{\G[{[v]}]'}$. 
We construct another $\G[{[v]}]''$ by adding an extra edge $e$ connecting two depth-$L$ nodes in $\G[{[v]}]'$. 
Edge $e$ changes the degree distribution $\delta_{[v]}\paren{\cdot}$, and thus $\tau\paren{\G[{[v]}]'}\neq\tau\paren{\G[{[v]}]''}$. 
On the other hand, there is no way for GraphSAGE to propagate the influence of edge $e$ to the target $v$, unless the model performs at least $L+1$ message passings. 
So $\zeta\paren{\G[{[v]}]'}=\zeta\paren{\G[{[v]}]''}$ regardless of the activation function and weight parameters. 
Therefore, $\zeta\neq\tau$. 

For {\shadowsage}, let $\sample$ return $\G[{[v]}]$. 
Then the model can output $\zeta'=\left[\Arw_{[v]}^{L'}\X\right]_{v,:}$ after we
\begin{enumerate*}
\item set $\bm{W}_1^{\paren{\ell}}=\bm{0}$ and $\bm{W}_2^{\paren{\ell}}=\bm{I}$ for all layers, and 
\item either remove the non-linear activation or bypass $\relu$ by shifting $\X$ with bias. 
\end{enumerate*}
With known results in Markov chain convergence theorem \cite{mc_mixing}, we derive the following theorem by analyzing the convergence of $\Arw_{[v]}^{L'}$ when $L'\rightarrow\infty$. 

\begin{theorem}
\label{thm: shadowsage tau}
{\shadowsage} can approximate $\tau$ with error decaying exponentially with depth. 
\end{theorem}

We have the following conclusions from above: 
\begin{enumerate*}
\item {\shadowsage} is more expressive than GraphSAGE. 
\item appropriate {\sample} function improves {\shadow} expressivity,
\item There exists cases where it may be desirable to set the {\shadow} depth much larger than the subgraph depth. 
\end{enumerate*}

\subsection{Expressivity Analysis on {\shadowgin}: Topological Learning Perspective}
\label{sec: shadow-gin}
While GCN and GraphSAGE are popular architectures in practice, they are not the theoretically most discriminative ones. 
The work in \cite{gin} establishes the relation in discriminativeness between GNNs and 1-dimensional Weisfeiler-Lehman test (\ie, 1-WL). 
And GIN \cite{gin} is an example architecture achieving the same discriminativeness as 1-WL. 
We show that applying the decoupling principle can further improve the discriminativeness of such GNNs, making them more powerful than 1-WL. 

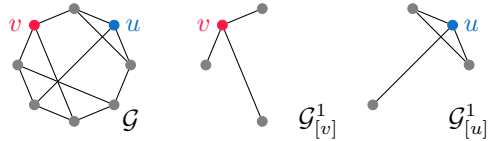
\begin{wrapfigure}{r}{0.49\textwidth}
  \vspace{-15pt}
  \begin{center}
  \input{diagrams/example_3regular_graph.tex}
  \vspace{-5pt}
  \captionof{figure}{Example 3-regular graph and the 1-hop subgraphs of the target nodes $u$ and $v$. }
  \label{fig: shadowgin example}
  \end{center}
  \vspace{-15pt}
\end{wrapfigure}
1-WL is a graph isomorphism test aiming at distinguishing graphs of different structures. 
A GNN as expressive as 1-WL thus well captures the topological property of the target node. 
While 1-WL is already very powerful, it may still fail in some cases. \eg, it cannot distinguish certain non-isomorphic, \emph{regular} graphs. 
To understand why {\shadow} works, we first need to understand why 1-WL fails. 
In a regular graph, all nodes have the same degree, and thus the ``regular'' property describes a \emph{global} topological symmetry among nodes. 
Unfortunately, 1-WL (and the corresponding normal GNN) also operates \emph{globally} on $\G$. 
Intuitively, on two different regular graphs, there is no way for 1-WL (and the normal GNN) to assign different labels by breaking such symmetry. 

On the other hand, {\shadow} can break such symmetry by applying decoupling. In Section \ref{sec: intro}, we have discussed how {\shadow} is built from the \emph{local} perspective on the full graph. 
The key property benefiting {\shadow} is that \emph{a subgraph of a regular graph may not be regular}. 
Thus, {\shadow} can distinguish nodes in a regular graph with the non-regular subgraphs as the scope. 
We illustrate the intuition with the example in Figure \ref{fig: shadowgin example}. 
The graph $\G$ is 3-regular and we assume all nodes have identical features. 
Our goal is to discriminate nodes $u$ and $v$ since their neighborhood structures are different. 
No matter how many iterations 1-WL runs, or how many layers the normal GNN has, they cannot distinguish $u$ and $v$. 
On the other hand, a {\shadow} with 1-hop {\sample} and at least 2 layers can discriminate $u$ and $v$. 

\begin{theorem}
\label{thm: shallow-shallow}
Consider GNNs whose layer function is defined by
\begin{align}
\label{eq: thm shallow-shallow}
    \vh_v^{(\ell)} = f_1^{(\ell)}\paren{\vh_v^{(\ell-1)},\; \sum_{u \in\mathcal{N}_v} f_2^{(\ell)}\paren{\vh_v^{(\ell-1)}, \vh_u^{(\ell-1)}}},
\end{align}
where $f_1^{(\ell)}$ and $f_2^{(\ell)}$ are the update and message functions of  layer-$\ell$, implemented as MLPs. 
Then, such {\shadow} is more discriminative than the 1-dimensional Weisfeiler-Lehman test. 
\end{theorem}

The theorem also implies that {\shadowgin} is more discriminative than a normal GIN due to the correspondence between GIN and 1-WL. 
See Appendix \ref{appendix: proof} for the proof of all theorems in Section \ref{sec: method}. 

\subsection{Subgraph Extraction Algorithms}
\label{sec: sampler}

Our decoupling principle does not rely on specific subgraph extraction algorithms. 
Appropriate {\sample} can be customized given the characteristic of $\G$, 
and different {\sample} leads to different \emph{implementation} of our decoupling principle. 
In general, we summarize three approaches to design {\sample}: 
\begin{enumerate*}
\item \emph{heuristic based}, where we pick graph metrics that reflect neighbor importance and then design {\sample} by such metrics; 
\item \emph{model based}, where we assume a generation process on $\G$ and set {\sample} as the reverse process, and 
\item \emph{learning based}, where we integrate the design of {\sample} as part of the GNN training. 
\end{enumerate*}
In the following, we present several examples on heuristic based {\sample}, which we also empirically evaluate in Section \ref{sec: exp}. 
We leave detailed evaluation on the model based and learning based {\sample} as future work. See also Appendix \ref{appendix: sampler detail} for details. 

\noindent\textbf{Example heuristic based {\sample}.}\hspace{.3cm}
The algorithm is derived from the selected graph metrics. 
For example, with the metric being shortest path distance, we design a $L$-hop extractor. \ie, {\sample} returns the full set or randomly selected subset of the target node's $L$-hop neighbors in $\G$. 
Picking the random walk landing probability as the metric, we can design a PPR-based extractor. 
\ie, we first run the Personalized PageRank (PPR) algorithm on $\G$ to derive the PPR score of other nodes relative to the target node. Then {\sample} define $\V[{[v]}]$ by picking the top-$K$ nodes with the highest PPR scores. 
The subgraph $\G[{[v]}]$ is the \emph{node-induced subgraph}\footnote{Unlike other PPR-based models \cite{gdc, pprgo} which rewire the graph by treating top PPR nodes as direct neighbors, our PPR neighborhood preserves the original multi-hop topology by returning node-induced subgraph. } of $\G$ from $\V[{[v]}]$. 
One can easily extend this approach by using other metrics such as Katz index \cite{katz}, SimRank \cite{simrank} and feature similarity. 

\subsection{Architecture}
\label{sec: ensemble}

\noindent\textbf{Subgraph pooling.}\hspace{.3cm} 
For a normal GNN performing node classification, the multi-layer message passing follows a ``tree structure''. 
The nodes at level $L$ of the tree correspond to the $L$-hop neighborhood. And the tree root outputs the final embedding of the target node. 
Thus, there is no way to apply subgraph pooling or READOUT on the final layer output, since the ``pool'' only contains a single vector. 
For a {\shadow}, since we decouple the $L$\textsuperscript{th} layer from the $L$-hop neighborhood, it is natural to let each layer (including the final layer) output embeddings for all subgraph nodes. 
This leads to the design to READOUT all the subgraph node embeddings as the target node embedding.

We can understand the pooling for {\shadow} from another perspective. In a normal GNN, the target node at the final layer receives messages from all neighbors, but two neighbor nodes may not have a chance to exchange any message to each other (\eg, two nodes $L$-hop away from the target may be $2L$-hop away from each other). 
In our design, a {\shadow} can pass messages between \emph{any} pair of neighbors when the model depth is large enough. 
Therefore, all the subgraph node embeddings at the final layer capture meaningful information of the neighborhood. 

In summary, the power of the decoupling principle lies in that \emph{it establishes the connection between the node- / link-level task and the graph-level task}. 
\eg, to classify a node is seen as to classify the subgraph surrounding the node. 
From the neural architecture perspective, we can apply any subgraph pooling / READOUT operation originally designed for graph classification (\eg, \cite{sortpool, self_pool, spec_pool}) to enhance the node classification / link prediction of {\shadow}. 
In particular, in the vanilla {\shadow}, we can implement a trivial READOUT as 
``discarding all neighbor embeddings'', corresponding to performing center pooling. 
See Appendix \ref{appendix: pool defn} and \ref{appendix: exp pool}  for algorithm and experiments.

\noindent\textbf{Subgraph ensemble.}\hspace{.3cm}
It may be challenging in practice to design a single {\sample} capturing all meaningful characteristics of the neighborhood. 
We can use multiple {\sample} to jointly define the receptive field, 
and then ensemble multiple {\shadow} at the subgraph level. 
Consider $R$ candidates $\set*{{\sample}_i}$,  each  returning $\G[{[v]}]^i$. 
To generate $v$'s embedding, we first use $R$ branches of $L'$-layer GNN to obtain intermediate embeddings for each $\G[{v}]^i$, and then aggregate the $R$ embeddings by some learnable function $g$. 
In practice, we design $g$ as an attention based aggregation function (see Appendix \ref{appendix: arch ensemble}). 
Subgraph ensemble is useful both when $\set*{{\sample}_i}$ consists of different algorithms and when each ${\sample}_i$ performs the same algorithm under different parameters. 

\textsc{Case study}\hspace{0.3cm}
Consider PPR-based ${\sample}_i$ with different threshold $\theta_i$ on the neighbor PPR score. 
A {\shadow}-ensemble can approximate PPRGo \citep{pprgo}. 
PPRGo generates embedding as: $\bm{\xi}_v=\sum_{u\in\V[{[v]}]}\pi_u\bm{h}_v$, where $\pi_u$ is $u$'s PPR score and $\bm{h}_v=\func[MLP]{\bm{x}_v}$. 
We can partition $\V[{[v]}]=\bigcup_{i=1}^R \V[{[v]}]^i$ s.t. nodes in $\V[{[v]}]^i$ have similar PPR scores denoted by $\widetilde{\pi}_i$, and $\widetilde{\pi}_i \leq \widetilde{\pi}_{i+1}$. 
So $\bm{\xi}_v=\sum_{i=1}^R\rho_i\paren{\sum_{u\in\V[i]'}\bm{h}_u}$, where $\rho_i=\widetilde{\pi}_i - \sum_{j<i}\widetilde{\pi}_j$ and $\V[i]'=\bigcup_{k=i}^R\V[{[v]}]^k$. 
Now for each branch of {\shadow}-ensemble, let parameter $\theta_i=\widetilde{\pi}_i$ so that ${\sample}_i$ returns $\V_i'$. 
The GNN on $\V_i'$ can then learn $\sum_{u\in\V_i'}\bm{h}_u$ (\eg, by a simple ``mean'' READOUT). 
Finally, set the ensemble weight as $\rho_i$. {\shadow}-ensemble learns $\bm{\xi}_v$. 
As {\sample} also preserves graph topology, our model can be more expressive than PPRGo.

\subsection{Practical Design: {\shadow}}
\label{sec: shadow practical}

We now discuss the practical implementation of decoupled GNN -- {\shadow}. 
As the name suggests, in {\shadow}, the scope is a shallow subgraph (\ie, with depth often set to $2$ or $3$). 

In many realistic scenarios (\eg, citation networks, social networks, product recommendation graphs), a shallow neighborhood is both \emph{necessary} and \emph{sufficient} for the GNN to learn well. 
On ``sufficiency'', we consider the social network example: the friend of a friend of a friend may share little commonality with you, and close friends may be at most 2 hops away. 
Formally, by the $\gamma$-decaying theorem \citep{seal}, a shallow neighborhood is sufficient to accurately estimate various graph metrics. 
On ``necessity'', since the neighborhood size may grow exponentially with hops, 
a deep neighborhood would be dominated by nodes irrelevant to the target. 
The corresponding GNN would first need to differentiate the many useless nodes from the very few useful ones, before it can extract meaningful features from the useful nodes. 
Finally, a shallow subgraph ensures scalability by avoiding ``neighborhood explosion''. 

\parag{Remark on decoupling.}
So far we have defined a decoupled model as having the model depth $L'$ larger than the subgraph depth $L$. 
Strictly speaking, a decoupled model also admits $L'=L$. 
For example, suppose in the full $L$-hop neighborhood, there are 70\% nodes $L$ hops away. 
Applying decoupling, the {\sample} excludes most of the $L$-hop neighbors, and the resulting subgraph $\G[{[v]}]$ contains only 20\% nodes $L$ hops away. 
Then it is reasonable to consider an $L$-layer model on such a depth-$L$ subgraph as also a decouple model. 
Compared with an $L$-layer model on the full $L$-hop neighborhood, an $L$-layer model on such a depth-$L$ $\G[{[v]}]$ propagates much less information from nodes $L$ hops away. 
So the $L$ message passings are indeed decoupled from the full $L$-hop neighborhood. 

\parag{Remark on neighborhood.}
The ``sufficiency'' and ``necessity'' in shallow neighborhood are not universal. 
In many other applications, long-range dependencies can be critical, as studied in \cite{bottleneck}. 
In such cases, our practical implementation of {\shadow} would incur accuracy loss. However, our decoupling principle in general may still be beneficial -- ``shallow subgraph'' is a practical guideline rather than a theoretical requirement. 
We leave the study on such applications as future work. 

%% file: diagrams/example_3regular_graph.tex
\definecolor{colCC1}{HTML}{0F6FCA}
\definecolor{colCC2}{RGB}{255,23,68}

\tikzset{
    archnode/.style={thick,ellipse,minimum height=2cm,minimum width=3cm,c4,draw=c4,fill=c0},
}
\begin{tikzpicture}[
  outer/.style={draw=gray,dashed,fill=green!1,thick},
  >={Stealth[inset=0pt,
  length=4pt,angle'=45,round]},scale=1.,
  sty edge sel/.style={ultra thick,color=c4},
  sty node sel/.style={ultra thick,circle,c4,draw=c4,fill=c0,minimum size=\sizenode},
  sty node unsel/.style={circle,c2,draw=c2,fill=c0,minimum size=\sizenode},
  sty edge unsel/.style={color=c2},
]

\def\R{0.75};
\def\offXa{-1.};
\def\offYa{0};
\draw ({\R*cos(0)+\offXa}, {\R*sin(0)+\offYa}) \foreach \theta in {45,90,...,360} {  -- ({\R*cos(\theta)+\offXa}, {\R*sin(\theta)+\offYa}) };
\foreach \theta in {45,90,...,360} {
    \ifthenelse{\NOT \theta = 45 \AND \NOT \theta = 135}{
      \node[inner sep=1.3pt,circle,draw=gray,fill=gray] (n\theta) at ({\R*cos(\theta)+\offXa}, {\R*sin(\theta)+\offYa}) {};
    }{
      \ifthenelse{\theta = 45} {
        \node[inner sep=1.3pt,circle,draw=colCC1,fill=colCC1] (n\theta) at ({\R*cos(\theta)+\offXa}, {\R*sin(\theta)+\offYa}) {};
      }{
        \node[inner sep=1.3pt,circle,draw=colCC2,fill=colCC2] (n\theta) at ({\R*cos(\theta)+\offXa}, {\R*sin(\theta)+\offYa}) {};
      }
    }
}

\draw (n45) -- (n225);
\draw (n90) -- (n360);
\draw (n315) -- (n180);
\draw (n270) -- (n135);

\node[color=colCC1] at ({\R*cos(45)+\offXa+0.25}, {\R*sin(45)+\offYa}) {$u$};
\node[color=colCC2] at ({\R*cos(135)+\offXa-0.25}, {\R*sin(135)+\offYa}) {$v$};

\node at ({\R*cos(270)+\offXa+0.75}, {\R*sin(270)+\offYa}) {$\G$};

\def\R{0.75};
\def\offXu{3.5};
\def\offYu{0};
\foreach \theta in {45,90,360,225} {
    \ifthenelse{\NOT \theta = 45}{
      \node[inner sep=1.3pt,circle,draw=gray,fill=gray] (u\theta) at ({\R*cos(\theta)+\offXu}, {\R*sin(\theta)+\offYu}) {};
    }{
      \node[inner sep=1.3pt,circle,draw=colCC1,fill=colCC1] (u\theta) at ({\R*cos(\theta)+\offXu}, {\R*sin(\theta)+\offYu}) {};
    }
}
\node[color=colCC1] at ({\R*cos(45)+\offXu+0.25}, {\R*sin(45)+\offYu}) {$u$};
\draw (u45) -- (u360);
\draw (u45) -- (u90);
\draw (u360) -- (u90);
\draw (u45) -- (u225);

\node at ({\R*cos(270)+\offXu+0.75}, {\R*sin(270)+\offYu}) {$\G[{[u]}]^1$};

\def\R{0.75};
\def\offXv{1.5};
\def\offYv{0};
\foreach \theta in {90,135,180,270} {
    \ifthenelse{\NOT \theta = 135}{
      \node[inner sep=1.3pt,circle,draw=gray,fill=gray] (v\theta) at ({\R*cos(\theta)+\offXv}, {\R*sin(\theta)+\offYv}) {};
    }{
      \node[inner sep=1.3pt,circle,draw=colCC2,fill=colCC2] (v\theta) at ({\R*cos(\theta)+\offXv}, {\R*sin(\theta)+\offYv}) {};
    }
}
\node at ({\R*cos(270)+\offXv+0.75}, {\R*sin(270)+\offYv}) {$\G[{[v]}]^1$};
\node[color=colCC2] at ({\R*cos(135)+\offXv-0.25}, {\R*sin(135)+\offYv}) {$v$};
\draw (v135) -- (v90);
\draw (v135) -- (v180);
\draw (v135) -- (v270);

\end{tikzpicture}

%% file: 4_related_work.tex
\section{Related Work}
\label{sec: related}

\noindent\textbf{Deep GNNs.}\hspace{.3cm}
To improve the GNN performance while increasing the model depth, various layer architectures have been proposed. 
AS-GCN \citep{asgcn}, DeepGCN \citep{deepgcn}, JK-net \citep{jknet}, MixHop \citep{mixhop}, Snowball \citep{break_the_ceiling}, DAGNN \citep{kdd_deep} and GCNII \cite{gcnii} all include some variants of residue connection, either across multiple layers or within a single layer. 
In principle, such architectures can also benefit the feature propagation of a deep {\shadow}, since their design does not rely on a specific neighborhood (\eg, $L$-hop). 
In addition to architectures, DropEdge \citep{dropedge} and Bayesian-GDC \cite{graph_drop_connect} propose regularization techniques by adapting dropout \citep{dropout} to graphs. Such techniques are only applied during training, and inference may still suffer from issues such as oversmoothing. 

\noindent\textbf{Sampling based methods.}\hspace{.3cm}
Neighbor or subgraph sampling techniques have been proposed to improve training efficiency. 
FastGCN \cite{fastgcn}, VR-GCN \cite{vrgcn}, AS-GCN \cite{asgcn}, LADIES \cite{ladies} and MVS-GNN \cite{mvs} sample neighbor nodes per GNN layer. 
Cluster-GCN \cite{clustergcn} and GraphSAINT \cite{graphsaint} sample a subgraph as the training minibatch. 
While sampling also changes the receptive field, all the above methods are fundamentally different from ours. 
The training samplers aim at estimating the quantities related to the full graph (\eg, the aggregation of the full $L$-hop neighborhood), and so the inference model still operates on the full neighborhood to avoid accuracy loss. 
For {\shadow}, since the decoupling principle is derived from a local view on $\G$, our {\sample} does not estimate any full neighborhood quantities. 
Consequently, the sampling based methods only improve the training efficiency, while {\shadow} addresses the computation challenge for both training and inference. 

\noindent\textbf{Re-defining the neighborhood.}\hspace{.3cm}
Various works reconstruct the original graph and apply the GNN on the re-defined neighborhood. 
GDC \cite{gdc} views the reconstructed adjacency matrix as the diffusion matrix. SIGN \cite{sign} applies reconstruction for customized graph operators. 
AM-GCN \cite{amgcn} utilizes the reconstruction to separate feature and structure information. 
The above work re-define the neighborhood. However, they still have tightly coupled depth and scope. 
{\shadow} can also work with the reconstructed graph $\G'$ by simply applying {\sample} on $\G'$. 

%% file: 5_exp.tex
\section{Experiments}
\label{sec: exp}
\noindent\textbf{Setup.}\hspace{.3cm}
We evaluate {\shadow} on seven graphs. 
Six of them are for the node classification task: {\flickr} \citep{graphsaint}, {\reddit}
\citep{graphsage}, {\yelp} \citep{graphsaint}, {\arxiv}, {\products} and {\papers} \cite{ogb}. 
The remaining is for the link prediction task: {\collab} \cite{ogb}. 
The sizes of the seven graphs range from 9K nodes ({\flickr}) to  110M nodes ({\papers}). {\flickr}, {\reddit} and {\yelp} are under the inductive setting. {\arxiv}, {\products} and {\papers} are transductive. 
Consistent with the original papers, for the graphs on node classification, we measure ``F1-micro'' score for {\yelp} and ``accuracy'' for the remaining five graphs. 
For the link prediction task, we use ``Hits@50'' as the metric. 
See Appendix \ref{appendix: dataset} for details. 

We construct \textsc{shaDow} with six backbone models: GCN \citep{gcn_welling}, GraphSAGE \citep{graphsage}, GAT \citep{gat}, JK-Net \citep{jknet}, GIN \citep{gin}, SGC \citep{sgc}. 
The first five are representatives of the state-of-the-art GNN architectures, jointly covering various message aggregation functions as well as the skip connection design. 
SGC simplifies normal GCN by moving all the neighbor propagation to the pre-processing step. Therefore, SGC is suitable for evaluating oversmoothing. 
The non-\textsc{shaDow} models are trained with both full-batch and GraphSAINT-minibatch \cite{graphsaint}. 
Due to the massive size of the full $L$-hop neighborhood, we need to perform sampling when training normal GNNs in order to make the computation time tractable. 
GraphSAINT is suitable for our purpose since 
\begin{enumerate*}
\item it is the state-of-the-art minibatch method which achieves high accuracy, and
\item it supports various GNN architectures and scales well to large graphs. 
\end{enumerate*}
On the other hand, for {\shadow}, both training and inference are always executed in minibatches. 
One advantage of {\shadow} is that the decoupling enables straightforward minibatch construction: each target just independently extracts the small subgraph on its own. 

We implement two {\sample} described in Section \ref{sec: sampler}: 
\begin{enumerate*}
\item ``PPR'', where we set the node budget $K$ as $\{200, 400\}$ for the largest {\papers} and $K\leq 200$ for all other graphs;  
\item ``$L$-hop'', where we set the depth as $\{1, 2\}$. 
\end{enumerate*}
We implement various subgraph pooling functions: ``mean'' and ``max'' evaluated in this section and others evaluated in Appendix \ref{appendix: exp pool}. 
For the model depth, since $L'=3$ is the standard setting in the literature (\eg, see the benchmarking in OGB \cite{ogb}), we start from $L'=3$ and further evaluate a deeper model of $L'=5$. 
All accuracy are measured by 5 runs without fixing random seeds. 
Hyperparameter tuning and architecture configurations are in Appendix \ref{appendix: hyperparameter}.

\begin{figure}
    \centering
    \begin{minipage}{0.6\textwidth}
        \centering
        \input{plots/hop_ratio.tex}
        \vspace{-.3cm}
        \captionof{figure}{Neighborhood composition}
        \label{fig: hop count}
    \end{minipage}
    \hfill
    \begin{minipage}{0.35\textwidth}
        \centering
        \input{plots/sgc_deep}
        \vspace{-.3cm}
        \captionof{figure}{SGC oversmoothing}
        \label{fig: sgc deep}
    \end{minipage}
    \vspace{-.5cm}
\end{figure}

\noindent\textbf{{\shadow} neighborhood.}\hspace{.3cm}
For both normal and \textsc{shaDow} GNNs, 
Figure \ref{fig: hop count} shows on average how many neighbors are $L$ hops away from the target. For a normal GNN, the size of the neighborhood grows rapidly with respect to $L$, and the nodes 4 hops away dominate the neighborhood. 
For {\shadow} using the Table \ref{tab: exp-sota} {\sample}, most neighbors concentrate within 2 hops. A small number of nodes are 3 hops away. Almost no nodes are 4 or more hops away. 
Importantly, not only does the composition of the two kinds of neighborhood differ significantly, but also the size of {\shadow} scope is much smaller (see also Table \ref{tab: exp-sota}). 
Such small subgraphs are essential to high computation efficiency. 
Finally, we can ignore the very few distant neighbors ($L\geq 4$), and regard the (effective) depth of {\shadow} subgraph as $L=2$ (or at most $3$). 
Under such practical value of $L$, a model with $L'\geq 3$ is indeed a {\shadow} (see ``Remark on decoupling'' in Section \ref{sec: shadow practical}). 

\input{./data/table1.tex}

\noindent\textbf{Comparison with baselines.}\hspace{.3cm}
Table \ref{tab: exp-sota} shows the performance comparison of {\shadow} with the normal GNNs. 
All models on all datasets have uniform hidden dimension of $256$. 
We define the metric ``inference cost'' as the average amount of computation to generate prediction for one test node. 
Inference cost is a measure of computation complexity (see Appendix \ref{appendix: inf cost calculation} for the calculation) and is independent of hardware / implementation factors such as parallelization strategy, batch processing, \textit{etc}. 
For cost of {\shadow}, we do not include the overhead on computing $\sample$, since it is hard to calculate such cost analytically. 
Empirically, subgraph extraction is much cheaper than model computation (see Figure \ref{fig: ppr time}, \ref{fig: acc-time tradeoff} for time evaluation on CPU and GPU). 
During training, we apply DropEdge \citep{dropedge} to both the baseline and \textsc{shaDow} models. DropEdge helps improve the baseline accuracy by alleviating oversmoothing, and benefits {\shadow} due to its regularization effects. See Appendix \ref{appendix: other arch} for results on other architectures including GIN and JK-Net. 

\textsc{Accuracy}\hspace{0.3cm} 
We aim at answering the following questions: 
\begin{enumerate*}
\item Can we improve accuracy by decoupling a baseline model?
\item What architecture components can we tune to improve accuracy of a decoupled model?
\item What {\sample} can we tune to improve the accuracy of a decoupled model?
\end{enumerate*}
To answer Q1, we fix the backbone architecture and remove pooling. Then we inspect ``3-layer normal GNN \emph{vs.} 3-layer {\shadow}'' and ``5-layer normal GNN \emph{vs.} 5-layer {\shadow}''.  
Clearly, {\shadow}s (with scope size no more than 200) in general achieve significantly higher accuracy than the normal GNNs. 
This indicates that a shallow neighborhood contains sufficient information, and customizing the scope can benefit accuracy even without architecture changes (from Figure \ref{fig: hop count}, a depth-3 $\G[{[v]}]$ differs significantly from the 3-hop neighborhood). 
To answer Q1, we focus on the PPR {\sample} and thus compare the rows in blue background. 
We use the 3-layer {\shadow} without pooling as the baseline and analyze the effects of 
\begin{enumerate*}
\item increasing the GNN depth without expanding scope, and
\item adding subgraph pooling. 
\end{enumerate*}
Comparing among the rows in \colorbox{LightCyan}{light blue} background, we observe that in many cases, simply increasing the depth from 3 to 5 leads to significant accuracy gain. 
Comparing the \colorbox{LightCyan}{ligh blue} rows with the \colorbox{HeavyCyan}{dark blue} rows, we observe that sometimes pooling can further improve the accuracy of a {\shadow}. 
In conclusion, both types of architecture tuning are effective ways of optimizing a {\shadow}. 
Finally, to answer Q3, we compare the \colorbox{LightCyan}{light blue} rows with the \colorbox{LightYellow}{light yellow} rows. 
In general, PPR {\sample} leads to higher accuracy than $2$-hop {\sample}, demonstrating the importance of designing a good {\sample}. 

\textsc{Inference cost}\hspace{0.3cm}
Inference cost of {\shadow} is orders of magnitude lower than the normal GNNs (a 5-layer {\shadow} is still much cheaper than a 3-layer normal GNN). 
The high cost of the baselines is due to the ``neighborhood explosion'' with respect to more layers. 
{\shadow} is efficient and scalable as the cost only grows linearly with the model depth. 
Note that GraphSAINT only improves efficiency \emph{during training} since its inference operates on the full $L$-hop neighborhood. 

\begin{wraptable}{r}{7.3cm}
\vspace{-15pt}
\caption{Leaderboard comparison on {\papersshort}}
\vspace{-.3cm}
\resizebox{0.52\textwidth}{!}{
\input{data/table_papers100M}
}
\label{tab: exp-papers}
\vspace{-.2cm}
\end{wraptable}
\noindent\textbf{Scaling to 100 million nodes.}\hspace{.3cm}
We further scale {\shadow} to {\papers}, one of the largest public dataset. 
Even through the full graph size is at least two orders of magnitude larger than the graphs in Table \ref{tab: exp-sota}, the localized scope of {\shadow} barely needs to increase. 
Since {\shadow} performs minibatch computation, a low-end GPU with limited memory capacity can compute {\shadow} on {\papers} efficiently. 
We show in Appendix \ref{appendix: exp mem speed} that we can train and inference our model with as little as 4GB GPU memory consumption. This is infeasible using normal GNNs. 
Table \ref{tab: exp-papers} summarizes our comparison with the top leaderboard methods \cite{dgl,sign,sgc}. 
We only include those methods that do not use node labels as the model input (\ie, the most standard setup). 
We achieve at least 3 orders of magnitude reduction in neighborhood size without sacrificing accuracy. 
For SIGN-XL and SGC, their neighborhood is too large to count the exact size. 
Also, their preprocessing consumes $5\times$ more CPU memory than {\shadow} (Appendix \ref{appendix: exp mem speed}).

\begin{wraptable}{r}{6.2cm}
\vspace{-15pt}
\caption{Leaderboard comparison on {\collabshort}}
\vspace{-.3cm}
\resizebox{0.45\textwidth}{!}{
\input{data/table_collab}
}
\label{tab: exp-collab}
\vspace{-.2cm}
\end{wraptable}
\noindent\textbf{Extending to link-level task.}\hspace{.3cm}
We further show that {\shadow} is general and can be extended to the link prediction task. 
There are two settings of {\collab}. We follow the one where validation edges cannot be used in training updates. This is the setting which most leaderboard methods follow. 
Table \ref{tab: exp-collab} shows the comparison with the top GNN models under the same setting. 
{\shadowsage} outperforms the \emph{rank-1} model with significant margin.

\noindent\textbf{Oversmoothing.}\hspace{.3cm}
To validate Theorem \ref{thm: non-oversmoothing}, we pick SGC as the backbone architecture. 
SGC with power $L$ is equivalent to $L$-layer GCN without activation. 
Performance comparison between SGC and {\shadowsgc} thus reveals the effect of oversmoothing without introducing other factors due to optimizing deep neural networks (e.g., vanishing gradients). 
In Figure \ref{fig: sgc deep},  we vary the power of SGC and {\shadowsgc} from 1 to 40 (see Appendix \ref{appendix: deep setup} for details). 
While SGC gradually collapses local information into global ``white noise'', accuracy of {\shadowsgc} does not degrade. 
This validates our theory that extracting local subgraphs prevents oversmoothing. 

%% file: plots/hop_ratio.tex
\tikzset{mark size=1}
\definecolor{c1}{HTML}{E2D726}
\definecolor{c2}{HTML}{15BD12}
\definecolor{c3}{HTML}{0F6FCA}
\definecolor{c4}{HTML}{B31090}

\def\plotwidth{0.3\textwidth}
\def\plotheight{0.25\textwidth}
\def\posylabelx{-0.2}
\def\posylabely{0.25}
\def\barwidth{5}
\def\xrotate{90}

\begin{tikzpicture}

\begin{groupplot}[
    group style={
        group size=2 by 1,
        y descriptions at=edge left,
        horizontal sep=2.5mm
    },
    scale only axis,
    height=\plotheight,
    width=\plotwidth,
    tick label style={font=\footnotesize},
    title style={font=\small,at={(axis description cs:0.5, 0.95)}},
    grid,
    label style = {font = {\fontsize{9.5 pt}{12 pt}\selectfont}},
    tick label style = {font = {\fontsize{8.5 pt}{12 pt}\selectfont}},
    ylabel={\small Ratio of nodes at different hop $\ell$},
    y label style={at={(axis description cs:\posylabelx, \posylabely)},anchor=south},
    symbolic x coords={{\flickr}, {\reddit}, {\yelp}, {\arxivshort}, {\productsshort}, {\papersshort}},
    legend style={font=\small},
    legend cell align=left,
    legend style={at={(1.3,0.9)},anchor=north,/tikz/column 4/.style={column sep=5pt},draw=none,/tikz/every even column/.append style={column sep=0.1cm}},
]

\nextgroupplot[
    enlarge x limits=0.2,ymin=0,ymax=1,
    ybar stacked,
    tick align=inside,
    xtick=data,
    x tick label style={rotate=\xrotate},
    bar width=\barwidth pt,
    ymajorgrids = true,
    yticklabel style={/pgf/number format/precision=2,/pgf/number format/fixed},
    title={4-layer GNN},
]

\addplot+[ybar,style={c1,fill=c1,mark=none}]
            plot coordinates {
                ({\flickr},.0002) 
                ({\reddit},.000) 
                ({\yelp},0.0) 
                ({\arxivshort},.0104) 
                ({\productsshort}, 0.0)
            };
\addplot+[ybar,style={c2,fill=c2,mark=none}]
             plot coordinates {
                ({\flickr}, .014) 
                ({\reddit},.0298) 
                ({\yelp},.0045) 
                ({\arxivshort},.0665)
                ({\productsshort},0.008)
            };  
\addplot+[ybar,style={c3,fill=c3,mark=none}]
             plot coordinates {
                ({\flickr}, .166) 
                ({\reddit},.3737) 
                ({\yelp},.133) 
                ({\arxivshort},.2625)
                ({\productsshort}, 0.108)
            };
\addplot+[ybar,style={c4,fill=c4,mark=none}]
             plot coordinates {
                ({\flickr}, .819) 
                ({\reddit},.5961) 
                ({\yelp},.862) 
                ({\arxivshort},.660)
                ({\productsshort},0.884)
            };

\nextgroupplot[
    enlarge x limits=0.2,ymin=0,ymax=1,
    ybar stacked,
    tick align=inside,
    xtick=data,
    x tick label style={rotate=\xrotate},
    bar width=\barwidth pt,
    ymajorgrids = true,
    yticklabel style={/pgf/number format/precision=2,/pgf/number format/fixed},
    title={Any layer \textsc{shaDow} (PPR)},
    title style={font=\small,at={(axis description cs:0.65, 0.95)}},
]
\addplot+[ybar,style={c1,fill=c1,mark=none}]
            plot coordinates {
                ({\flickr},.07) 
                ({\reddit},.409) 
                ({\yelp},0.160) 
                ({\arxivshort},.107) 
                ({\productsshort}, 0.69)
            };
\addplot+[ybar,style={c2,fill=c2,mark=none}]
             plot coordinates {
                ({\flickr}, .701) 
                ({\reddit},.557) 
                ({\yelp},.505) 
                ({\arxivshort},.658)
                ({\productsshort},0.292)
            };  
\addplot+[ybar,style={c3,fill=c3,mark=none}]
             plot coordinates {
                ({\flickr}, .209) 
                ({\reddit},.1) 
                ({\yelp},.301) 
                ({\arxivshort},.207)
                ({\productsshort},0.02)
            };
\addplot+[ybar,style={c4,fill=c4,mark=none}]
             plot coordinates {
                ({\flickr}, .093) 
                ({\reddit},.1) 
                ({\yelp},.03) 
                ({\arxivshort},.05)
                ({\productsshort},0.0)
            };
            
\legend{$\ell=1$,$\ell=2$,$\ell=3$,$\ell=4$}

\end{groupplot}
\end{tikzpicture}

%% file: plots/sgc_deep.tex
\definecolor{colA}{RGB}{255,23,68}  
\definecolor{colB}{HTML}{B6A807}    
\definecolor{colC}{HTML}{4AAC2B}    
\definecolor{colD}{HTML}{0F6FCA}    
\definecolor{skyblue}{RGB}{213,0,0}
\definecolor{yellowgreen}{RGB}{255,23,68}
\definecolor{red1}{HTML}{FB7E78}
\definecolor{c3}{HTML}{0F6FCA}

\def\plotwidth{1.0\textwidth}
\def\plotheight{.8\textwidth}
\def\maxtime{250}
\def\minacc{80}
\def\plotW{0.59}

\def\sizemark{1.5}

\begin{tikzpicture}
\begin{axis}[
    height=\plotheight,
    width=\plotwidth,
    grid,
    label style = {font = {\fontsize{8. pt}{12 pt}\selectfont}},
    tick label style = {font = {\fontsize{8.5 pt}{12 pt}\selectfont}},
    ylabel near ticks,
    xlabel near ticks,
    enlarge x limits=false,
    xlabel=Power on $\widetilde{\bm{A}}$ or $\widetilde{\bm{A}}_{[v]}$,
    ylabel=Test accuracy,
    xlabel shift=-4pt,
    ylabel shift=-3pt,
    ymin=0.4, ymax=1,
    xmin=0, xmax=45,
    legend cell align=left,
    legend columns=3,
    legend style={
        draw=none,
        fill=none,
        nodes={scale=0.75,},
        at={(1.15, 1.35)},
        /tikz/every even column/.append style={column sep=0.07cm},
    },
]
\errorband [draw=colA, fill=colA,mark=*, mark size=\sizemark pt]{\datatable}{0}{1}{2}{
    \textsc{s}-SGC, {\fontfamily{lmtt}\selectfont F}
} 

\errorband [draw=colB, fill=colB,mark=*, mark size=\sizemark pt]{\datatable}{0}{5}{6}{
    \textsc{s}-SGC, {\fontfamily{lmtt}\selectfont R}
}

\errorband [draw=colC, fill=colC,mark=*, mark size=\sizemark pt]{\datatable}{0}{9}{10}{
    \textsc{s}-SGC, {\fontfamily{lmtt}\selectfont A}
}

\errorband [draw=colA, fill=colA, dashed,mark=*, mark size=\sizemark pt]{\datatable}{0}{3}{4}{
    SGC, {\fontfamily{lmtt}\selectfont F}
}

\errorband [draw=colB, fill=colB, dashed,mark=*, mark size=\sizemark pt]{\datatable}{0}{7}{8}{
    SGC, {\fontfamily{lmtt}\selectfont R}
}

\errorband [draw=colC, fill=colC, dashed,mark=*, mark size=\sizemark pt]{\datatable}{0}{11}{12}{
    SGC, {\fontfamily{lmtt}\selectfont A}
}

\node[color=gray] at (10, 0.87) {\scriptsize {\fontfamily{lmtt}\selectfont R}eddit};
\node[color=gray] at (25, 0.75) {\scriptsize {\fontfamily{lmtt}\selectfont A}rxiv};
\node[color=gray] at (13, 0.47) {\scriptsize {\fontfamily{lmtt}\selectfont F}lickr};

\end{axis}
\end{tikzpicture}

%% file: data/table1.tex
\begin{table*}[!ht]
\caption{Comparison on test accuracy / F1-micro score and inference cost (tuned with DropEdge)}
    \centering
    \resizebox{1\textwidth}{!}{
    \begin{tabular}{lccccccccccc}
        \toprule
        \multirow{2}{*}{Method} & \multirow{2}{*}{Layers} & \multicolumn{2}{c}{\flickr} & \multicolumn{2}{c}{\reddit} & \multicolumn{2}{c}{\yelp} & \multicolumn{2}{c}{\arxiv} & \multicolumn{2}{c}{\products}\\
        \cmidrule(lr){3-4}\cmidrule(lr){5-6}\cmidrule(lr){7-8}\cmidrule(lr){9-10}\cmidrule(lr){11-12}& & Accuracy & Cost & Accuracy & Cost & F1-micro & Cost & Accuracy & Cost & Accuracy & Cost\\
        \midrule
        \midrule
         \multirow{2}{*}{GCN} & 3 &
            0.5159\std{0.0017} & 2E0 & 
            0.9532\std{0.0003} & 6E1 & 
            0.4028\std{0.0019} & 2E1 & 
            0.7170\std{0.0026} & 1E1 & 
            0.7567\std{0.0018}& 5E0\\
         & 5 & 
            0.5217\std{0.0016} & 2E2 & 
            0.9495\std{0.0012} & 1E3 & 
            OOM & 1E3 & 
            0.7186\std{0.0017} & 1E3 & 
            OOM & 9E2 \\
         \multirow{2}{*}{GCN-SAINT} & 3 & 
            0.5155\std{0.0027} & 2E0 & 
            0.9523\std{0.0003} & 6E1 & 
            0.5110\std{0.0012} & 2E1 & 
            0.7093\std{0.0003} & 1E1 & 
            \textbf{0.8003}\std{0.0024} & 5E0\\
         & 5 & 0.5165\std{0.0026} & 2E2 &
            0.9523\std{0.0012} & 1E3 & 
            0.5012\std{0.0021} & 1E3 & 
            0.7039\std{0.0020} & 1E3 & 
            0.7992\std{0.0021} & 9E2\\
        \cmidrule(lr){1-12}\cellcolor{LightCyan} & 3 & 
            0.5234\std{0.0009} & \textbf{(1)} & 
            0.9576\std{0.0005} & \textbf{(1)} & 
            0.5291\std{0.0020} & \textbf{(1)} & 
            0.7180\std{0.0024} & \textbf{(1)} & 
            0.7742\std{0.0037} & \textbf{(1)}\\
         \cellcolor{LightCyan}\multirow{-2}{*}{\shortstack[l]{{\shadowgcn}\\{(PPR)}}} & 5 & 
            0.5268\std{0.0008} & 1E0 & 
            0.9564\std{0.0004} & 1E0 & 
            0.5323\std{0.0020} & 2E0 & 
            0.7206\std{0.0025} & 2E0 & 
            0.7821\std{0.0043} & 2E0 \\
        \cellcolor{HeavyCyan}{+Pooling} & 3/5 & 
            \textbf{0.5286}\std{0.0013} & 1E0 &
            \textbf{0.9624}\std{0.0002} & 1E0 &
            \textbf{0.5393}\std{0.0036} & 2E0 &
            \textbf{0.7223}\std{0.0018} & 2E0 &
            0.7914\std{0.0044} & 2E0 \\
        \midrule
        \midrule
         \multirow{2}{*}{GraphSAGE} & 3 & 
            0.5140\std{0.0014} & 3E0 & 
            0.9653\std{0.0002} & 5E1 & 
            0.6178\std{0.0033} & 2E1 & 
            0.7192\std{0.0027} & 1E1 & 
            0.7858\std{0.0013} & 4E0 \\
         & 5 & 
            0.5154\std{0.0052} & 2E2 & 
            0.9626\std{0.0004} & 1E3 & 
            OOM &2E3 & 
            0.7193\std{0.0037} & 1E3 & 
            OOM & 1E3 \\
         \multirow{2}{*}{SAGE-SAINT} & 3 & 
            0.5176\std{0.0032} & 3E0 & 
            0.9671\std{0.0003} & 5E1 & 
            0.6453\std{0.0011} & 2E1 & 
            0.7107\std{0.0003} & 1E1 & 
            0.7923\std{0.0023} & 4E0\\
         & 5 & 
            0.5201\std{0.0032} & 2E2 &
            0.9670\std{0.0010} & 1E3 & 
            0.6394\std{0.0003} & 2E3 & 
            0.7013\std{0.0021} & 1E3 & 
            0.7964\std{0.0022} & 1E3\\
        \cmidrule(lr){1-12}\cellcolor{LightYellow} & 3 & 
            0.5312\std{0.0019} & 1E0 & 
            0.9672\std{0.0003} & 1E0 & 
            0.6542\std{0.0002} & 1E0 & 
            0.7163\std{0.0028} & 1E0 & 
            0.7935\std{0.0031} & 1E0\\
         \cellcolor{LightYellow}\multirow{-2}{*}{\shortstack[l]{{\shadowsage}\\{($2$-hop)}}}& 5 & 
            0.5335\std{0.0015} & 2E0 & 
            0.9675\std{0.0005} & 2E0 & 
            0.6525\std{0.0003} & 2E0 & 
            0.7180\std{0.0030} & 2E0 & 
            0.7995\std{0.0022} & 2E0\\
        \cellcolor{LightCyan} & 3 &
            0.5356\std{0.0013} & \textbf{(1)} & 
            0.9688\std{0.0002} & \textbf{(1)} & 
            0.6538\std{0.0003} & \textbf{(1)} & 
            0.7227\std{0.0012} & \textbf{(1)} & 
            0.7905\std{0.0026} & \textbf{(1)}\\
         \cellcolor{LightCyan}\multirow{-2}{*}{\shortstack[l]{{\shadowsage}\\{(PPR)}}} & 5 & 
            \textbf{0.5417}\std{0.0006} & 2E0 & 
            0.9692\std{0.0007} & 2E0 & 
            0.6518\std{0.0002} & 2E0 & 
            0.7238\std{0.0007} & 2E0 & 
            0.8005\std{0.0040} & 2E0 \\
        \cellcolor{HeavyCyan}{+Pooling} & 3/5 & 
            0.5395\std{0.0013} & 2E0 &
            \textbf{0.9703}\std{0.0003} & 2E0 &
            \textbf{0.6564}\std{0.0004} & 1E0 &
            \textbf{0.7258}\std{0.0017} & 2E0 &
            \textbf{0.8067}\std{0.0037} & 1E0\\
        \midrule
        \midrule
         \multirow{2}{*}{GAT} & 3 & 
            0.5070\std{0.0032} & 2E1 & 
            OOM & 3E2 & 
            OOM & 2E2 & 
            0.7201\std{0.0011} & 1E2 & 
            OOM & 3E1\\
         & 5 &
            0.5164\std{0.0033} & 2E2 & 
            OOM & 2E3 & 
            OOM & 2E3 & 
            OOM & 3E3 & 
            OOM & 2E3 \\
         \multirow{2}{*}{GAT-SAINT} & 3 & 
            0.5225\std{0.0053} & 2E1 & 
            0.9671\std{0.0003} & 3E2 & 
            0.6459\std{0.0002} & 2E2 & 
            0.6977\std{0.0003} & 1E2 & 
            0.8027\std{0.0028} & 3E1\\
         & 5 & 0.5153\std{0.0034} & 2E2 &
            0.9651\std{0.0024} & 2E3 & 
            0.6478\std{0.0012} & 2E3 & 
            0.6954\std{0.0013} & 3E3 & 
            0.7990\std{0.0072} & 2E3\\
        \cmidrule(lr){1-12}\cellcolor{LightCyan} & 3 & 
            0.5349\std{0.0023} & \textbf{(1)} & 
            0.9707\std{0.0004} & \textbf{(1)} & 
            \textbf{0.6575}\std{0.0004} & \textbf{(1)} & 
            0.7235\std{0.0020} & \textbf{(1)} & 
            0.8006\std{0.0014} & \textbf{(1)}\\
        \cellcolor{LightCyan}\multirow{-2}{*}{\shortstack[l]{{\shadowgat}\\{(PPR)}}} & 5 & 
            0.5352\std{0.0028} & 2E0 & 
            \textbf{0.9713}\std{0.0004} & 2E0 & 
            0.6559\std{0.0002} & 2E0 & 
            \textbf{0.7274}\std{0.0022} & 2E0 & 
            0.8071\std{0.0004} & 2E0\\
        \cellcolor{HeavyCyan}{+Pooling} & 3/5 & 
            \textbf{0.5364}\std{0.0026} & 1E0 &
            0.9710\std{0.0004} & 2E0 &
            0.6566\std{0.0005} & 1E0 &
            0.7265\std{0.0028} & 2E0 &
            \textbf{0.8142}\std{0.0031} & 1E0 \\
        \bottomrule
    \end{tabular}
    }
    \label{tab: exp-sota}
\end{table*}

%% file: data/table_papers100M.tex
\begin{tabular}{lcccc}
    \toprule
    Method & Test accuracy & Val accuracy & Neigh size\\
    \midrule
    \midrule
    GraphSAGE+incep & 0.6706\std{0.0017} & 0.7032\std{0.0011} & 4E5\\
    SIGN-XL & 0.6606\std{0.0019} & 0.6984\std{0.0006} & $>$ 4E5\\
    SGC & 0.6329\std{0.0019} & 0.6648\std{0.0020} & $>$ 4E5\\
    \midrule
    {\shadowgat}\textsubscript{200} & 0.6681\std{0.0016} & 0.7019\std{0.0011} & 2E2\\
    {\shadowgat}\textsubscript{400} & \textbf{0.6708}\std{0.0017} & \textbf{0.7073}\std{0.0011} & 3E2\\
    \bottomrule
\end{tabular}

%% file: data/table_collab.tex
\begin{tabular}{lccc}
    \toprule
    Method & Test Hits@50 & Val Hits@50\\
    \midrule
    \midrule
    SEAL & 0.5371\std{0.0047} & 0.6389\std{0.0049}\\
    DeeperGCN & 0.5273\std{0.0047} & 0.6187\std{0.0045}\\
    LRGA+GCN & 0.5221\std{0.0072} & 0.6088\std{0.0059}\\
    \midrule
    {\shadowsage} & \textbf{0.5492}\std{0.0022} & \textbf{0.6524}\std{0.0017}\\
    \bottomrule
\end{tabular}

%% file: 6_conclusion.tex
\section{Conclusion}

We have presented a design principle to decouple the depth and scope of GNNs. 
Applying such a principle on various GNN architectures simultaneously improves expressivity and computation scalability of the corresponding models. 
We have presented thorough theoretical analysis on expressivity from three different perspectives, 
and also rich design components (\eg, subgraph extraction functions, architecture extensions) to implement such design principle. 
Experiments show significant performance improvement over a wide range of graphs, GNN architectures and learning tasks.

%% file: 7_appendix.tex
\input{appendix/proof}

\section{Inference Complexity Calculation}
\label{appendix: inf cost calculation}
Here we describe the equations to compute the ``inference cost'' of Table \ref{tab: exp-sota}. 
Recall that inference cost is a measure of computation complexity to generate node embeddings for a given GNN architecture. 

The numbers in Table \ref{tab: exp-sota} shows on average, how many arithmetic operations is required to generate the embedding for each node. For a GNN layer $\ell$, denote the number of input nodes as $n^{(\ell-1)}$ and the number of output nodes as $n^{(\ell)}$. Denote the number of edges connecting the input and output nodes as $m^{(\ell)}$. 
Recall that we use $d^{(\ell)}$ to denote the number of channels, or, the hidden dimension. 

In the following, we ignore the computation cost of non-linear activation, batch-normalization and applying bias, since their complexity is negligible compared with the other operations. 

For the \emph{GCN} architecture, each layer mainly performs two operations: aggregation of the neighbor features and linear transformation by the layer weights. So the number of multiplication-addition (MAC) operations of a layer-$\ell$ equals:

\begin{align}
\label{eq: complexity gcn}
    P_{\text{GCN}}^{(\ell)} = m^{(\ell)} d^{(\ell-1)} + n^{(\ell)} d^{(\ell-1)} d^{(\ell)}
\end{align}

Similarly, for the \emph{GraphSAGE} architecture, the number of MAC operations equals:

\begin{align}
\label{eq: complexity sage}
    P_{\text{SAGE}}^{(\ell)} = m^{(\ell)} d^{(\ell-1)} + 2\cdot n^{(\ell)} d^{(\ell-1)} d^{(\ell)}
\end{align}

where the $2$ factor is due to the weight transformation on the self-features. 

For \emph{GAT}, suppose the number of attention heads is $t$. Then the layer contains $t$ weight matrices $\bm{W}^i$, each of shape $d^{(\ell-1)}\times \frac{d^{(\ell)}}{t}$. 
We first transform each of the $n^{(\ell-1)}$ nodes by $\bm{W}^i$. Then for each edge $(u,v)$ connecting the layer input $u$ to the layer output $v$, we obtain its edge weight (\textit{i.e.}, a scalar) by computing the dot product between $u$'s, $v$'s transformed feature vectors and the model attention weight vector. 
After obtaining the edge weight, the remaining computation is to aggregate the $n^{(\ell-1)}$ features into the $n^{(\ell)}$ nodes. 
The final output is obtained by concatenating the features of different heads. 
The number of MAC operations equals:

\begin{align}
\label{eq: complexity gat}
    P_{\text{GAT}}^{(\ell)} =& t\cdot n^{(\ell-1)}d^{(\ell-1)}\frac{d^{(\ell)}}{t} + 2t\cdot m^{(\ell)}\frac{d^{(\ell)}}{t} + m^{(\ell)}d^{(\ell)} \nonumber\\
    =&3m^{(\ell)}d^{(\ell)} + n^{(\ell-1)}d^{(\ell-1)}d^{(\ell)}
\end{align}

On the same graph, GCN is less expensive than GraphSAGE. GraphSAGE is in general less expensive than GAT, due to $n^{(\ell-1)}>n^{(\ell)}$ (on the other hand, note that $n^{(\ell-1)} = n^{(\ell)}$ for any {\shadow}). 
In addition, for all architectures, $P_{*}^{(\ell)}$ grows proportionally with $n^{(\ell)}$ and $m^{(\ell)}$. For the normal GNN architecture, since we are using the full $\ell$-hop neighborhood for each node, the value of $n^{(\ell)}$ and $m^{(\ell)}$ may grow exponentially with $\ell$. This is the ``neighbor explosion'' phenonemon and is the root cause of the high inference cost of the baseline in Table \ref{tab: exp-sota}. 

For {\shadow}, suppose the subgraph contains $n$ nodes and $m$ edges. Then $n^{(\ell)}=n$ and $m^{(\ell)}=m$. The inference cost of any {\shadow} is ensured to only grow linearly with the depth of the GNN. 

\parag{Remark}
The above calculation on normal GNNs is under the realistic setting of minibatch inference. 
By ``realistic setting'', we mean 
\begin{enumerate*}
\item the graph size is often gigantic (of similar scale as {\papers});
\item the total number of target nodes can be much, much smaller than the full graph size (\eg, $<1\%$ of all nodes are target nodes, as in {\papers}), and
\item we may only want to generate embedding for a small subset of the target nodes at a time. 
\end{enumerate*}
As a result, the above calculated computation complexity reflects what we can achieve in practice, under minibatch computation. 

On the other hand, most of the benchmark graphs evaluated in the paper may not reflect the realistic setting. For example, {\arxiv} is downscaled $657\times$ from {\papers}. Consequently, the full graph size is small and the number of target nodes is comparable to the full graph size.  
Such a ``benchmark setting'' enables full-batch inference: \eg, for GCN, one multiplication on the full graph adjacency matrix (\emph{i.e.}, $\Asym\cdot \bm{H}\cdot \bm{W}$) generates the next layer hidden features for all nodes at once. 

While full-batch computation leads to significantly lower computation complexity than Equations \ref{eq: complexity gcn}, \ref{eq: complexity sage} and \ref{eq: complexity gat}, it has strict constraints on the GPU memory size and the graph size -- the full adjacency matrix and the various feature data of \emph{all} nodes need to completely fit into the GPU memory. For example, according to the official OGB guideline \citep{ogb}, full-batch training of {\products} on a 3-layer GNN requires 48GB of GPU memory. On one hand, only the most powerful NVIDIA GPUs (\emph{e.g.}, RTX A6000) have memory as large as 48GB. On the other hand, {\products} is still $45\times$ smaller than {\papers}.
Therefore, full batch computation is not feasible in practice. 

In summary, the inference complexity shown in Table \ref{tab: exp-sota} based on Equations \ref{eq: complexity gcn}, \ref{eq: complexity sage} and \ref{eq: complexity gat} reflects the feasibility of deploying the GNN models in \emph{real-life applications}. 

\input{appendix/sampler_details}

\input{appendix/complete_framework}

\section{Detailed Experimental Setup}

We used the tool ``Weight \& Biases'' \cite{wandb} for experiment tracking and visualization to develop insights for this paper. 

\subsection{Additional Dataset Details}
\label{appendix: dataset}
The statistics for the seven benchmark graphs is listed in Table \ref{tab: datasets}. Note that for {\yelp}, each node may have multiple labels, and thus we follow the original paper \citep{graphsaint} to report its F1-micro score. For all the other node classification graphs, a node is only associated with a single label, and so we report accuracy. 
Note that for {\reddit} and {\flickr}, other papers \citep{graphsage, graphsaint} also report F1-micro score. However, since each node only has one label, ``F1-micro score'' is exactly the same as ``accuracy''. 
For the link prediction graph {\collab}, we use ``Hits@50'' as the metric. 

For {\papers}, only around $1\%$ of the nodes are associated with ground truth labels. The training / validation / test sets are split only among those labelled nodes. 

\begin{table*}[!ht]
\caption{Dataset statistics}
    \centering
    \resizebox{\textwidth}{!}{
    \begin{tabular}{lccccccc}
        \toprule
        Dataset & Setting & Nodes & Edges & Degree & Feature & Classes & Train / Val / Test\\
        \midrule
        \midrule
        {\flickr} & Inductive & 89,250 & 899,756 & 10 & 500 & 7 & 0.50 / 0.25 / 0.25\\
        {\reddit} & Inductive & 232,965 & 11,606,919 & 50 & 602 & 41 & 0.66 / 0.10 / 0.24\\
        {\yelp} & Inductive & 716,847 & 6,977,410 & 10 & 300 & 100 & 0.75 / 0.10 / 0.15 \\
        {\arxiv} & Transductive & 169,343 & 1,166,243 & 7 & 128 & 40 & 0.54 / 0.18 / 0.29 \\
        {\products} & Transductive & 2,449,029 & 61,859,140 & 25 & 100 & 47 & 0.10 / 0.02 / 0.88\\
        {\papers} & Transductive & 111,059,956 & 1,615,685,872 & 29 & 128 & 172 & 0.78 / 0.08 / 0.14\\
        \midrule
        {\collab} & -- & 235,868 & 1,285,465 & 8 & 128 & -- & 0.92/0.04/0.04\\
        \bottomrule
    \end{tabular}}
    \label{tab: datasets}
\end{table*}

Note that {\reddit} is a pre-exisiting dataset collected by Stanford, available at \url{http://snap.stanford.edu/graphsage}. Facebook did not directly collect any data from Reddit. None Reddit content is reproduced in this paper.

\subsection{Hardware}
\label{appendix: hw spec}

We have tested {\shadow} under various hardware configurations, ranging from low-end desktops to high-end servers. We observe that the training and inference of {\shadow} can be easily adapted to the amount of available hardware resources by adjusting the batch size. 

In summary, we have used the following three machines to compute {\shadow}. 

\begin{itemize}
\item \textsc{Machine 1}: This is a low-end desktop machine with 4-core Intel Core i7-6700 CPU @3.4GHz, 16GB RAM and one NVIDIA GeForce GTX 1650 GPU of 4GB memory.
\item \textsc{Machine 2}: This is a low-end server with 28-core Intel Xeon Gold 5120 CPU @2.2GHz, 128GB RAM and one NVIDIA Titan Xp GPU of 12GB memory.
\item \textsc{Machine 3}: This is another server with 64-core AMD Ryzen Threadripper 3990x CPU @2.9GHz, 256GB RAM and three NVIDIA GeForce RTX3090 GPUs of 24GB memory. 
\end{itemize}

From the GPU perspective, the low-end GTX 1650 GPU can support the {\shadow} computation on all the graphs (including {\papers}). However, the limited RAM size of \textsc{Machine 1} limits its usage to only {\flickr}, {\reddit}, {\arxiv} and {\collab}. The other two servers, \textsc{Machine 2} and \textsc{Machine 3}, are used to train all of the seven graphs.

Table \ref{tab: min hardware} summarizes our recommended minimum hardware resources to run {\shadow}. Note that regardless of the model, larger graphs \emph{inevitibly requires larger RAM} due to the growth of the raw features (the raw data files of {\papers} already takes 70GB). However, larger graphs \emph{do not correspond to higher GPU requirement} for {\shadow}, since the GPU memory consumption is controlled by the batch size parameter. 

See Appendix \ref{appendix: exp mem speed} for how to control the GPU memory-speed tradeoff by batch size.

\begin{table}[!ht]
    \caption{Recommended minimum hardware resources for {\shadow}}
    \centering
    \begin{tabular}{lcccc}
        \toprule
        Dataset & Num. nodes & CPU cores & RAM & GPU memory\\
        \midrule\midrule
        {\arxiv} & 0.2M & 4 & 8GB & 4GB \\
        {\products} & 2.4M & 4 & 32GB & 4GB \\
        {\papers} & 111.1M & 4 & 128GB & 4GB\\
        \bottomrule
    \end{tabular}
    \label{tab: min hardware}
\end{table}

\subsection{Software}

The code is written in \verb|Python 3.8.5| (where the sampling part is written with \verb|C++| parallelized by \verb|OpenMP|, and the interface between \verb|C++| and \verb|Python| is via \verb|PyBind11|). We use \verb|PyTorch 1.7.1| on \verb|CUDA 11.1| to train the model on GPU.

\subsection{Hyperparameter Tuning}
\label{appendix: hyperparameter}
For all models in Table \ref{tab: exp-sota}, we set the hidden dimension to be $d^{(\ell)}=256$. In addition, for GAT and \textsc{shaDow-GAT}, we set the number of attention heads to be $t=4$. 
For all the GIN and \textsc{shaDow-GIN} experiments, we use a 2-layer MLP (with hidden dimension $256$) to perform the injective mapping in each GIN layer. For all the JK-Net and \textsc{shaDow-JK} experiments, we use the concatenation operation to aggregate the hidden features of each layer in the JK layer. 
Additional comparison on GIN, JK-Net and their \textsc{shaDow} versions is shown in Table \ref{tab: shadow other}. 

All experiments are repeated five times without fixing random seeds. 

For all the baseline and {\shadow} experiments, we use Adam optimizer \citep{adam}. We perform grid search on the hyperparameter space defined by:

\begin{itemize}
    \item Activation: $\left\{\text{\fontfamily{lmtt}\selectfont ReLU}, \text{\fontfamily{lmtt}\selectfont ELU}, \text{\fontfamily{lmtt}\selectfont PRELU}\right\}$
    \item Dropout: 0 to 0.5 with stride of 0.05
    \item DropEdge: 0 to 0.5 with stride of 0.05
    \item Learning rate: $\{\num{1e-2}, \num{2e-3}, \num{1e-3}, \num{2e-4}, \num{1e-4}, \num{2e-5} \}$
    \item Batch size: $\{16, 32, 64, 128\}$
\end{itemize}

The {\sample} hyperparameters are tuned as follows. 

For the PPR {\sample}, we consider two versions: one based on fixed sampling budget $p$ and the other based on PPR score thresholding $\theta$. 
\begin{itemize}
\item If with fixed budget, then we disable the $\theta$ thresholding. We tune the budget by $p\in\left\{150, 175, 200 \right\}$.
\item If with thresholding, we set $\theta\in\left\{0.01, 0.05\right\}$. We still have an upper bound $p$ on the subgraph size. So if there are $q$ nodes in the neighborhood with PPR score larger than $\theta$, the final subgraph size would be $\max\{p, q\}$. Such an upper bound eliminates the corner cases which may cause hardware inefficiency due to very large $q$. In this case, we set the upper bound $p$ to be either 200 or 500.
\end{itemize}

For $L$-hop {\sample}, we define the hyperparameter space as:
\begin{itemize}
    \item Depth $L=2$
    \item Budget $b\in\left\{5, 10, 15, 20\right\}$
\end{itemize}

The hyperparameters to reproduce the Table \ref{tab: exp-sota} {\shadow} results are listed in Tables \ref{tab: config} and \ref{tab: config 2}.  
The hyperparameters to reproduce the Table \ref{tab: exp-papers} and \ref{tab: exp-collab} {\shadow} results are listed in Table \ref{tab: config 3}. 

\begin{table*}
\caption{Training configuration of {\shadow} for Table \ref{tab: exp-sota} (PPR {\sample})}
    \centering
    \resizebox{1.0\textwidth}{!}{
    \begin{tabular}{llcccccccc}
        \toprule
        Arch. & Dataset & Layers & Dim. & Pooling & Learning Rate & Batch Size & Dropout & DropEdge & Budget $p$\\
        \midrule
        \midrule
        \multirow{15}{*}{\shadowgcn} & \multirow{3}{*}{\flickr} & 
            3 & 256 & -- & 0.001 & 256 & 0.40 & 0.10 & 200\\
        & & 5 & 256 & -- & 0.001 & 256 & 0.40 & 0.10 & 200\\
        & & 5 & 256 & mean & 0.001 & 256 & 0.40 & 0.10 & 200\\
        \cmidrule(lr){2-10} & \multirow{3}{*}{\reddit} & 
            3 & 256 & -- & 0.0001 & 128 & 0.20 & 0.15 & 200\\
        & & 5 & 256 & -- & 0.0001 & 128 & 0.20 & 0.15 & 200\\
        & & 3 & 256 & max & 0.0001 & 128 & 0.20 & 0.15 & 200\\
        \cmidrule(lr){2-10} & \multirow{3}{*}{\yelp} & 
            3 & 256 & -- & 0.001 & 32 & 0.10 & 0.00 & 100\\
        & & 5 & 256 & -- & 0.001 & 32 & 0.10 & 0.00 & 100\\
        & & 5 & 256 & max & 0.001 & 32 & 0.10 & 0.00 & 100\\
        \cmidrule(lr){2-10} & \multirow{3}{*}{\arxiv} & 
            3 & 256 & -- & 0.00005 & 32 & 0.20 & 0.10 & 200\\
        & & 5 & 256 & -- & 0.00005 & 32 & 0.20 & 0.10 & 200\\
        & & 5 & 256 & max & 0.00002 & 16 & 0.25 & 0.10 & 200\\
        \cmidrule(lr){2-10} & \multirow{3}{*}{\products} & 
            3 & 256 & -- & 0.002 & 128 & 0.40 & 0.05 & 150\\
        & & 5 & 256 & -- & 0.002 & 128 & 0.40 & 0.05 & 150\\
        & & 5 & 256 & max & 0.002 & 128 & 0.40 & 0.05 & 150\\
        \cmidrule(lr){1-10}\multirow{15}{*}{\shadowsage} & \multirow{3}{*}{\flickr} & 
            3 & 256 & -- & 0.0005 & 64 & 0.45 & 0.05 & 200\\
        & & 5 & 256 & -- & 0.001 & 128 & 0.45 & 0.00 & 200\\
        & & 5 & 256 & max & 0.001 & 128 & 0.45 & 0.00 & 200\\
        \cmidrule(lr){2-10} & \multirow{3}{*}{\reddit} & 
            3 & 256 & -- & 0.0001 & 128 & 0.20 & 0.15 & 200\\
        & & 5 & 256 & -- & 0.0001 & 128 & 0.20 & 0.15 & 200\\
        & & 5 & 256 & max & 0.0001 & 128 & 0.20 & 0.15 & 200\\
        \cmidrule(lr){2-10} & \multirow{3}{*}{\yelp} & 
            3 & 256 & -- & 0.0005 & 16 & 0.10 & 0.00 & 100\\
        & & 5 & 256 & -- & 0.0005 & 16 & 0.10 & 0.00 & 100\\
        & & 3 & 256 & max & 0.0005 & 16 & 0.10 & 0.00 & 100\\
        \cmidrule(lr){2-10} & \multirow{3}{*}{\arxiv} & 
            3 & 256 & -- & 0.00002 & 16 & 0.25 & 0.15 & 200\\
        & & 5 & 256 & -- & 0.00002 & 16 & 0.25 & 0.15 & 200\\
        & & 5 & 256 & max & 0.00002 & 16 & 0.25 & 0.15 & 200\\
        \cmidrule(lr){2-10} & \multirow{3}{*}{\products} & 
            3 & 256 & -- & 0.002 & 128 & 0.40 & 0.05 & 150\\
        & & 5 & 256 & -- & 0.002 & 128 & 0.40 & 0.05 & 150\\
        & & 3 & 256 & max & 0.002 & 128 & 0.40 & 0.15 & 150\\
        \cmidrule(lr){1-10}\multirow{15}{*}{\shadowgat} & \multirow{3}{*}{\flickr} & 
            3 & 256 & -- & 0.0005 & 64 & 0.45 & 0.00 & 200\\
        & & 5 & 256 & -- & 0.0005 & 64 & 0.45 & 0.00 & 200\\
        & & 3 & 256 & mean & 0.001 & 128 & 0.40 & 0.00 & 200\\
        \cmidrule(lr){2-10} & \multirow{3}{*}{\reddit} & 
            3 & 256 & -- & 0.0001 & 128 & 0.20 & 0.00 & 200\\
        & & 5 & 256 & -- & 0.0001 & 128 & 0.20 & 0.00 & 200\\
        & & 5 & 256 & max & 0.0001 & 128 & 0.20 & 0.00 & 200\\
        \cmidrule(lr){2-10} & \multirow{3}{*}{\yelp} & 
            3 & 256 & -- & 0.0005 & 16 & 0.10 & 0.00 & 100\\
        & & 5 & 256 & -- & 0.0005 & 16 & 0.10 & 0.00 & 100\\
        & & 3 & 256 & max & 0.0005 & 16 & 0.10 & 0.00 & 100\\
        \cmidrule(lr){2-10} & \multirow{3}{*}{\arxiv} & 
            3 & 256 & -- & 0.0001 & 64 & 0.20 & 0.00 & 200 \\
        & & 5 & 256 & -- & 0.0001 & 64 & 0.20 & 0.00 & 200\\
        & & 5 & 256 & max & 0.0001 & 64 & 0.20 & 0.00 & 200\\
        \cmidrule(lr){2-10} & \multirow{3}{*}{\products} & 
            3 & 256 & -- & 0.001 & 128 & 0.35 & 0.10 & 150\\
        & & 5 & 256 & -- & 0.001 & 128 & 0.35 & 0.10 & 150\\
        & & 3 & 256 & max & 0.001 & 128 & 0.35 & 0.10 & 150\\
        \bottomrule
    \end{tabular}}
    \label{tab: config}
\end{table*}

\begin{table*}
\caption{Training configuration of {\shadow} for Table \ref{tab: exp-sota} ($2$-hop {\sample})}
    \centering
    \resizebox{1.0\textwidth}{!}{
    \begin{tabular}{llcccccccc}
        \toprule
        Arch. & Dataset & Layers & Dim. & Pooling & Learning Rate & Batch Size & Dropout & DropEdge & Budget $b$\\
        \midrule
        \midrule
        \multirow{10}{*}{\shadowsage} & \multirow{2}{*}{\flickr} & 
            3 & 256 & -- & 0.0005 & 64 & 0.45 & 0.05 & 20\\
        & & 5 & 256 & -- & 0.0005 & 64 & 0.45 & 0.05 & 20\\
        \cmidrule(lr){2-10} & \multirow{2}{*}{\reddit} & 
            3 & 256 & -- & 0.0001 & 128 & 0.20 & 0.15 & 20\\
        & & 5 & 256 & -- & 0.0001 & 128 & 0.20 & 0.15 & 20\\
        \cmidrule(lr){2-10} & \multirow{2}{*}{\yelp} & 
            3 & 256 & -- & 0.0005 & 16 & 0.10 & 0.00 & 20\\
        & & 5 & 256 & -- & 0.0005 & 16 & 0.10 & 0.00 & 20\\
        \cmidrule(lr){2-10} & \multirow{2}{*}{\arxiv} & 
            3 & 256 & -- & 0.00005 & 16 & 0.25 & 0.15 & 20\\
        & & 5 & 256 & -- & 0.00005 & 16 & 0.25 & 0.15 & 20\\
        \cmidrule(lr){2-10} & \multirow{2}{*}{\products} & 
            3 & 256 & -- & 0.002 & 128 & 0.40 & 0.05 & 20\\
        & & 5 & 256 & -- & 0.002 & 128 & 0.40 & 0.05 & 20\\
        \bottomrule
    \end{tabular}}
    \label{tab: config 2}
\end{table*}

\begin{table*}
    \caption{Training configuration of {\shadow} for Table \ref{tab: exp-papers} and \ref{tab: exp-collab} (PPR {\sample})}
        \centering
        \resizebox{1.0\textwidth}{!}{
            \begin{tabular}{llcccccccc}
                \toprule
                Dataset & Arch. & Layers & Dim. & Pooling & Learning Rate & Dropout & DropEdge & Budget $p$ & Threshold $\theta$\\
                \midrule
                \midrule
                \multirow{2}{*}{\papers} 
                & {\shadowgat\textsubscript{200}} & 5 & 800 & max & 0.0002 & 0.30 & 0.10 & 200 & 0.002\\
                & {\shadowgat}\textsubscript{400} & 3 & 800 & max & 0.0002 & 0.35 & 0.10 & 400 & 0.002\\
                \cmidrule(lr){1-10}\multirow{1}{*}{\collab} & {\shadowsage} & 5 & 256 & sort & 0.00002 & 0.25 & 0.10 & 200 & 0.002\\ 
                \bottomrule
            \end{tabular}}
        \label{tab: config 3}
    \end{table*}

\subsection{Setup of the SGC Experiments}
\label{appendix: deep setup}

Following \cite{sgc}, we compute the SGC model as $\bm{Y}=\func[softmax]{\widetilde{\bm{A}}^K\bm{X}\bm{W}}$ where $\widetilde{\bm{A}} = \widetilde{\bm{D}}^{-\frac{1}{2}}\widetilde{\bm{A}}\widetilde{\bm{D}}^{-\frac{1}{2}}$ and $\widetilde{\bm{A}} = \bm{I}+\bm{A}$. Matrix $\bm{W}$ is the only learnable parameter. $K$ is the power on the adjacency matrix and we vary it as $K\in\{1, 3, 5, 10, 20, 30, 40\}$ in the Figure \ref{fig: sgc deep} experiments. 
For \textsc{shaDow-SGC}, we compute the embedding for target $v$ as $\bm{y}_v=\left[\func[softmax]{\widetilde{\bm{A}}_{[v]}^K\bm{X}_{[v]}\bm{W}}\right]_{v,:}$. 

SGC and {\shadowsgc} are trained with the same hyperparameters (\textit{i.e.}, learning rate equals 0.001 and dropout equals 0.1, across all datasets). {\shadowsgc} uses the same {\sample} as the {\shadowgcn} model in Table \ref{tab: exp-sota}. In the legend of Figure \ref{fig: sgc deep}, due to lack of space, we use \textsc{s-SGC} to denote {\shadowsgc}. We use ``{\fontfamily{lmtt}\selectfont F}'', ``{\fontfamily{lmtt}\selectfont R}'' and ``{\fontfamily{lmtt}\selectfont A}'' to denote the datasets of {\flickr}, {\reddit} and {\arxiv} respectively.

\section{Additional Experimental Results}
\label{appendix: additional exp}

\subsection{Understanding the Low Memory Overhead of {\shadow}}
\label{appendix: exp mem speed}

{\shadow} in general consumes much less memory than the other GNN-based methods. 
We can understand the low memory overhead of {\shadow} from two perspectives. 

Comparing with the normal GNN model, {\shadow} requires much less GPU memory since both the training and inference of {\shadow} proceeds in minibatches. 
While other minibatching training algorithms exist for the normal GNNs (\emph{e.g.}, layer sampling \citep{vrgcn} and graph sampling \citep{graphsaint}), such algorithms either result in accuracy degradation or limited scalability. Therefore, most of the OGB leaderboard methods are tuned with \emph{full-batch} training.  

Comparing with the simplified GNN models (\emph{e.g.}, SIGN \citep{sign} and SGC \citep{sgc}), {\shadow} requires much less RAM size. The reason is that both SIGN and SGC rely on preprocessing of node features over the full graph. Specifically, for SIGN, its preprocessing step concatenates the original node features with the smoothened values. For the $(p,s,t)$ architecture of SIGN (see the SIGN paper for definition of $p$, $s$ and $t$), the memory required to store the preprocessed features equals:

\begin{equation}
M = N \cdot f \cdot \paren{\paren{p+1} + \paren{s+1} + \paren{t+1}}
\end{equation}

where $N$ is the number of nodes in the full graph and $f$ is the original node feature size. For the $(3,3,3)$ SIGN / SIGN-XL architecture on the {\papers} leaderboard (see Table \ref{tab: exp-papers}), the required RAM size equals $M = 682$GB. This number does not even consider the RAM consumptions due to temporary variables or other buffers for the trainig operations. 

For SGC, even though it does not perform concatenation of the smoothened features, it still requires double the original feature size to store the temporary values of $\mA^K\cdot \bm{X}$. The original features of {\papers} takes around 56GB, and the full graph adjacency matrix consumes around 25GB. In sum, SGC requires at least $2\cdot 56 + 25 = 137$GB of RAM. 

Table \ref{tab: ogb measured mem} summarizes the RAM / GPU memory consumption for the leaderboard methods listed in Table \ref{tab: exp-papers}. 
Note that our machines do not support the training of SIGN (due to the RAM size constraint), and thus we only show the lower bound of SIGN's RAM consumption in Table \ref{tab: ogb measured mem}. 

\begin{table}[!ht]
    \caption{Memory consumption of the {\papers} leaderboard methods}
        \centering
        \begin{tabular}{lcc}
            \toprule
            Method & CPU RAM & GPU memory\\
            \midrule\midrule
            GraphSAGE+incep & 80GB & 16GB\\
            SIGN-XL & $>$682GB & 4GB \\
            SGC & $>$137GB & 4GB\\
            {\shadowgat} & \textbf{80GB} & \textbf{4GB}\\
            \bottomrule
        \end{tabular}
    \label{tab: ogb measured mem}
\end{table}

On the other hand, {\shadow} can flexibly adjust its batch size based on the available memory. Even for {\papers}, a typical low-end server with 4GB of GPU memory and 128GB of RAM is sufficient to train the 5-layer {\shadowgat}. Increasing the batch size of {\shadow} may further lead to higher GPU utilization for more powerful machines. 
Figure \ref{fig: exp mem time} shows the computation time speedup (compared with batch size 32) and GPU memory consumption for {\shadowgat} under batch size of 32, 64 and 128. 
A 5-layer {\shadowgat} only requires around 5GB of GPU memory to saturate the computation resources of the powerful GPU cards such as NVIDIA RTX 3090.

\begin{figure}
    \begin{center}
    \input{./plots/gpu_mem_speed.tex}
    \caption{Controlling the GPU memory-speed tradeoff by {\shadowgat} batch size (32, 64, 128)}
    \label{fig: exp mem time}
    \end{center}
\end{figure}

\subsection{{\shadow} on Other Architectures}
\label{appendix: other arch}
In addition to the GCN, GraphSAGE and GAT models in Table \ref{tab: exp-sota}, 
we further compare JK-Net and GIN with their \textsc{shaDow} versions in Table \ref{tab: shadow other}. For all the results in Tabel \ref{tab: exp-sota},  we do not apply pooling or ensemble on {\shadow}. 
Similar to DropEdge, the skip connection (or ``jumping knowledge'') of JK-Net helps accuracy improvement on deeper models. 
Compared with the normal JK-Net, increasing the depth benefits \textsc{shaDow-JK} more.  
The GIN architecture theoretically does not oversmooth. However, we observe that the GIN training is very sensitive to hyperparameter settings. 
We hypothesize that such a challenge is due to the sensitivity of the sum aggregator on noisy neighbors (\textit{e.g.}, for GraphSAGE, a single noisy node can hardly cause a significant perturbation on the aggregation, due to the averaging over the entire neighborhood). 
The accuracy improvement of {\shadowgin} compared with the normal GIN may thus be due to noise filtering by the {\sample}. 
The impact of noises / irrelevant neighbors can  be critical, as reflected by the 5-layer GIN accuracy on {\reddit}. 

\begin{table}[!ht]
    \caption{Test accuracy on other architectures (PPR {\sample})}
        \centering
        \resizebox{\linewidth}{!}{
        \input{data/table2_gin_jk}
        }
    \label{tab: shadow other}
\end{table}

\subsection{Benefits of Pooling}
\label{appendix: exp pool}

In Table \ref{tab: exp pool}, we summarize the average accuracy of 5-layer {\shadowsage} obtained from different pooling functions (see Table \ref{tab: pool ops} for the equations for pooling). 
For both graphs, adding a pooling layer helps improve accuracy. On the other hand, the best pooling function may depend on the graph characteristics. We leave the in-depth analysis on the effect of pooling as future work. 

\begin{table}[!ht]
    \caption{Effect of subgraph pooling on 5-layer {\shadowsage}}
        \centering
        \begin{tabular}{lcccc}
            \toprule
            Dataset & None & Mean & Max & Sort\\
            \midrule\midrule
            {\flickr} & 0.5351\std{0.0026} & 0.5361\std{0.0015} & 0.5354\std{0.0021} & \textbf{0.5367}\std{0.0026}\\
            {\arxiv} & 0.7302\std{0.0014} & 0.7304\std{0.0015}& \textbf{0.7342}\std{0.0017} & 0.7295\std{0.0018}\\
            \bottomrule
        \end{tabular}
    \label{tab: exp pool}
\end{table}

\subsection{Cost of Subgrpah Extraction} 

\begin{figure}
    \begin{center}
    \input{./plots/ppr_time.tex}
    \caption{Measured execution time for PPR {\sample} and the GNN model computation}
    \label{fig: ppr time}
    \end{center}
\end{figure}
We evaluate the PPR {\sample} in terms of its execution time overhead and accuracy-time tradeoff. 
In Figure \ref{fig: ppr time}, we parallelize {\sample} using half of the available CPU cores of \textsc{Machine 3} (\emph{i.e.}, 32 cores) and execute the GNN computation on the RTX 3090 GPU. 
Clearly, the PPR {\sample} is lightweight: the sampling time is lower than the GNN computation time in most cases. In addition, the sampling time per node does not grow with the full graph size. This shows that {\shadow} is scalable to massive scale graphs. 
By the discussion in Appendix \ref{appendix: sampler detail}, the approximate PPR computation achieves efficiency and scalability by only traversing a local region around each target node. 

\subsection{Tradeoff Between Inference Time and Accuracy}
\label{appendix: acc-time tradeoff}
\begin{figure}
    \centering
    \input{plots/acc_time_tradeoff.tex}
    \captionof{figure}{Inference performance tradeoff. We test pre-trained models by subgraphs of various sizes. }
    \label{fig: acc-time tradeoff}
\end{figure}

To evaluate the accuracy-time tradeoff, we take the $5$-layer models of Table \ref{tab: exp-sota} as the pretrained models. Then for {\shadow}, we vary the PPR budget $p$ from 50 to 200 with stride 50. 
In Figure \ref{fig: acc-time tradeoff}, the inference time of {\shadow} has already included the PPR sampling time. 
Firstly, consistent with Table \ref{tab: exp-sota}, inference of {\shadow} achieves higher accuracy than the normal GNNs, with orders of magnitude speedup as well. 
In addition, based on the application requirements (\textit{e.g.}, latency constraint), {\shadow}s have the flexibility of adjusting the sampling size without the need of retraining. 
For example, on {\reddit} and {\arxiv}, directly reducing the subgraph size from 200 to 50 reduces the inference latency by $2\times$ to $4\times$ at the cost of less than $1\%$ accuracy drop.

\begin{table}[!ht]
\caption{{\shadowgcn} test accuracy}
\centering
\input{data/table3_ensemble}
        \label{tab: exp shadow ens deep}
\end{table}

\subsection{Ensemble and Deeper Models}
\label{appendix: exp ens depth}

Table \ref{tab: exp shadow ens deep} shows additional results on subgraph ensemble and deeper {\shadowgcn} models. 
The PPR and $2$-hop {\sample} follow the same configuration as Table \ref{tab: config} and \ref{tab: config 2}. When varying the model depth, we keep all the other hyperparameters unchanged. 
From both the {\flickr} and {\products} results, we observe that ensemble of the PPR {\sample} and the $2$-hop {\sample} helps improve the {\shadowgcn} accuracy. 
From the {\products} results, we additionally observe that increasing {\shadowgcn} to deeper than 5 layers may still be beneficial. 
As analyzed by Figure \ref{fig: hop count}, the model depth of 5 is already much larger than the hops of the subgraph. The 7-layer results in Table \ref{tab: exp shadow ens deep} indicate future research opportunities to improve the {\shadow} accuracy by going even deeper.

%% file: appendix/proof.tex
\section{Proofs}
\label{appendix: proof}

\subsection{Proof on {\shadowgcn} Expressivity}
\begin{proof}[Proof of Proposition \ref{prop: shadow embedding}]
The GCN model performs symmetric normalization on the adjacency matrix. 
{\shadowgcn} follows the same way to normalize the subgraph adjacency matrix as:
\begin{align}
\Asym_{[v]} = \paren{\bm{D}_{[v]}+\bm{I}_{[v]}}^{-\frac{1}{2}}\cdot \paren{\bm{A}_{[v]}+\bm{I}_{[v]}}\cdot \paren{\bm{D}_{[v]}+\bm{I}_{[v]}}^{-\frac{1}{2}}
\end{align}
where $\bm{A}_{[v]}\in\mathbb{R}^{N\times N}$ is the binary adjacency matrix for $\G[{[v]}]$. 

$\Asym_{[v]}$ is a real symmetric matrix and has the largest eigenvalue of $1$. Since {\sample} ensures the subgraph $\G[{[v]}]$ is connected, so the multiplicity of the largest eigenvalue is 1. By Theorem 1 of \cite{dropedge_journal}, we can bound the eigenvalues $\lambda_i$ by $1 = \lambda_1 >\lambda_2 \geq \hdots \geq \lambda_N > -1$. 

Performing eigen-decomposition on $\Asym_{[v]}$, we have 
\begin{align}
    \Asym_{[v]} = \bm{E}_{[v]}\bm{\Lambda}\bm{E}_{[v]}^{-1} = \bm{E}_{[v]}\bm{\Lambda}\bm{E}_{[v]}^\trans
\end{align}

where $\bm{\Lambda}$ is a diagonal matrix $\Lambda_{i,i} = \lambda_i$ and matrix $\bm{E}_{[v]}$  consists of all the normalized eigenvectors. 

We have:
\begin{align}
    \Asym_{[v]}^L = \bm{E}_{[v]}\bm{\Lambda}^L\bm{E}_{[v]}^\trans
\end{align}

Since $\size{\lambda_i} < 1$ when $i\neq 1$, we have $\lim_{L\rightarrow \infty}\Asym_{[v]}^L = \bm{e}_{[v]}\bm{e}_{[v]}^\trans$, where $\bm{e}_{v}$ is the eigenvector corresponding to $\lambda_1$. 
It is easy to see that $\left[\bm{e}_{[v]}\right]_{u}\propto \sqrt{\delta_{[v]}(u)}$ \citep{dropedge_journal}. After normalization, $\left[\bm{e}_{[v]}\right]_{u} = \sqrt{\frac{\delta_{[v]}(u)}{\sum_{w\in\V[{[v]}]}\delta_{[v]}(w)}}$.

It directly follows that $\bm{m}_{[v]} = \left[e_{[v]}\right]_v\cdot \paren{\bm{e}_{[v]}^\trans\X_{[v]}}$, with value of $\bm{e}_{[v]}$ defined above. 
\end{proof}

\begin{proof}[Proof of Theorem \ref{thm: non-oversmoothing}]

We first prove the case of $\overline{\bm{m}}_{[v]}=\bm{m}_{[v]}$. \emph{i.e.}, $\phi_{\G}(v) = 1$. 

According to Proposition \ref{prop: shadow embedding}, the aggregation for each target node equals $\bm{m}_{[v]} = \left[e_{[v]}\right]_v\bm{e}_{[v]}^\trans\X_{[v]}$. 
Let $N=\size{\V}$. Let $\ee_{[v]}\in\mathbb{R}^{N\times 1}$ be a ``padded'' vector from $\bm{e}_{[v]}$, such that the $u^{\text{th}}$ element is $\left[e_{[v]}\right]_u$ if $u \in \V[{[v]}]$, and 0, otherwise. Therefore, $\bm{m}_{[v]} = \left[e_{[v]}\right]_v \ee_{[v]}^\trans\X$. Then, the difference in aggregation of two nodes $u$ and $v$ is given by
\begin{align}
    \bm{m}_{[v]} - \bm{m}_{[u]} 
    &= \left[e_{[v]}\right]_v \ee_{[v]}^\trans\X - \left[e_{[u]}\right]_u \ee_{[u]}^\trans\X \\
&= \bm{\epsilon}^\trans \X,
\end{align}

where
$
\bm{\epsilon} = \left[e_{[v]}\right]_{v}\cdot\ee_{[v]}^\trans - \left[e_{[u]}\right]_{u}\cdot \ee_{[u]}^\trans
$.

When two nodes $u$ and $v$ have identical neighborhood as $\V[{[u]}]=\V[{[v]}]$, then the aggregation vectors are identical as $\bm{\epsilon}=\bm{0}$. However, when two nodes have different neighborhoods, we claim that they almost surely have different aggregations. Let us assume the contrary, \emph{i.e.},
for some $v$ and $u$ with $\V[{[v]}]\neq \V[{[u]}]$, their aggregations are the same: $\bm{m}_{[v]}=\bm{m}_{[u]}$. 
Then we must have $\bm{\epsilon}^\trans \X=\bm{0}$. 

Note that, given $\G$, {\sample} and some continuous distribution to generate $\X$, there are only \emph{finite} values for $\bm{\epsilon}$. 
The reasons are that 
\begin{enumerate*}
    \item $\G$ is finite due to which there are only finite possible subgraphs and, 
    \item even though $\X$ can take infinitely many values, $\bm{\epsilon}$ does not depend on $\X$. 
\end{enumerate*}
Each of such $\bm{\epsilon}\neq \bm{0}$ defines a hyperplane in $\mathbb{R}^{N}$ by $\bm{\epsilon}\cdot \bm{x} = \bm{0}$ (where $\bm{x}\in\mathbb{R}^{N}$).
Let $\mathcal{H}$ be the finite set of all such hyperplanes. 

In other words, $\forall i, \;\X_{:, i}$ must fall on one of the hyperplanes in $\mathcal{H}$. 
However, since $\X$ is generated from a continuous distribution in $\mathbb{R}^{N\times f}$, $\X_{:, i}$ almost surely does not fall on any of the hyperplanes in $\mathcal{H}$. 
Therefore, for any $v$ and $u$ such that $\V[{[v]}]\neq \V[{[u]}]$, we have $\bm{m}_{[v]}\neq \bm{m}_{[u]}$ almost surely. 

For a more general $\phi_{\G}(v)$ applied on $\bm{m}_{[v]}$, since $\phi_{\G}$ does not depend on $\X$, the proof follows exactly the same steps as above. 

\end{proof}

\begin{proof}[Proof of Corollary \ref{coro: size-n sampler}]
Note that any subgraph contains the target node itself. \ie, $\forall u\in\V, \; u\in \V[{[u]}]$. 
Therefore, for any node $v$, there are at most $n-1$ other nodes in $\V$ with the same neighborhood as $\V[{[v]}]$. Such $n-1$ possible nodes are exactly those in $\V[{[v]}]$. 

By Theorem \ref{thm: non-oversmoothing}, $\forall v\in \V$, there are at most $n-1$ other nodes in $\V$ having the same aggregation as $\bm{m}_{[v]}$. 
Equivalently, total number of possible aggregations is at least $\ceil{\size{\V}/n}$. 
\end{proof}

\begin{proof}[Proof of Corollary \ref{coro: non-identical sampler}]
By definition of ${\sample}_2$, any pair of nodes has non-identical neighborhood. By Theorem \ref{thm: non-oversmoothing}, any pair of nodes have non-identical aggregation. Equivalently, all nodes have different aggregation and $\size{\mathcal{M}}=\size{\V}$. 
\end{proof}

\subsection{Proof on {\shadowsage} Expressivity}

In Section \ref{sec: shadow-sage}, we have already shown why the normal GraphSAGE model cannot learn the $\tau$ function. 
Here we illustrate the idea with an example in Figure \ref{fig: sage failure example}. 
The neighborhood $\G[{[v]}]'$ (or, $\G[{[v]}]''$) is denoted by the solid lines as the edge set $\E[{[v]}]'$ (or, $\E[{[v]}]''$) and solid nodes as the node set $\V[{[v]}]'$ (or, $\V[{[v]}]''$). 
The red dashed edge and the red node $v'$ is in the full graph $\G$ but outside $\G[{[v]}]'$ or $\G[{[v]}]''$. 
Due to such $v'$, candidate GraphSAGE models approximating $\tau$ cannot have message passing for more than 2 times. 
Further, $\G[{[v]}]''$ differs from $\G[{[v]}]'$ by the green edge $e$ connecting two 2-hop neighbors of the target $v$. 
The influence of $e$ cannot be propagated to target $v$ by only two message passings. 
Thus, as discussed in Section \ref{sec: shadow-sage}, there is no way for GraphSAGE to learn the $\tau$ function. 

\begin{figure}
    \centering
    \input{diagrams/deep_example_expand}
    \caption{Example $\G[{[v]}]'$ and $\G[{[v]}]''$ as described in Section \ref{sec: shadow-sage}}
    \label{fig: sage failure example}
\end{figure}

\begin{proof}[Proof of Theorem \ref{thm: shadowsage tau}]
{\shadowsage} follows the GraphSAGE way of normalization on the adjacency matrix. Thus, each row of $\Arw_{[v]}$ sums up to 1. 
Such normalization enables us to view $\Arw_{[v]}$ as the transition matrix of a Markov chain. The coefficients of the linear function $\tau$ (\ie, $\delta_{[v]}\paren{\cdot}$) equal the limiting distribution of such a Markov chain. 
Therefore, we can use the convergence theorem of Markov chain to characterize the convergence of $\Arw_{[v]}^{L'}$ towards $\tau$'s coefficients. 
We can also use the mixing time of the Markov chain to derive the error bound of shaDow-SAGE. 
Our proof is built on the existing theoretical results in \cite{mc_mixing}.

We rewrite $\tau$ as $\tau=C\cdot \bm{\delta}\bm{X}$, where $\bm{\delta}$ is a length-$\size{\V}$ vector with $\size{\V[{[v]}]}$ non-zero elements -- each non-zero corresponds to $\delta_{[v]}\paren{u}$ of the neighborhood $\G[{[v]}]$. 
Further, denote $\deltarw$ as the normalized $\bm{\delta}$ vector. \ie, each non-zero element of $\deltarw$ equals $\frac{\delta_{[v]}\paren{u}}{\sum_{w\in\V[{[v]}]}\delta_{[v]}\paren{w}}$. 
So $\tau=C'\cdot\deltarw\X$ with some scaling factor $C'$.
For ease of discussion, we ignore $C'$, as any scaling factor can be easily expressed by the model weight parameters. 

By Section \ref{sec: shadow-sage}, {\shadowsage} can express $\left[\Arw_{[v]}^{L'}\X\right]_{v,:}=\left[\Arw_{[v]}^{L'}\right]_{v,:}\X$. 
So now we need to show how $\left[\Arw_{[v]}^{L'}\right]_{v,:}$ converges to $\deltarw$ when $L'\rightarrow\infty$. 
We establish the following correspondence:

\parag{$\Arw_{[v]}$ as the Markov transition matrix.}
This can be easily seen since each row of $\Arw_{[v]}$ sums up to 1. 
Further, if we add self-loop in $\V[{[v]}]$ and we consider $\G[{[v]}]$ as undirected and connected (as can be guaranteed by {\sample}), then $\Arw_{[v]}$ is \emph{irreducible}, \emph{aperiodic} and \emph{reversible}. 
Irreducibility is guaranteed by $\G[{[v]}]$ being connected. Aperiodicity is guaranteed by self-loops. 
For reversibility, we can prove it by the stationary distribution of $\Arw_{[v]}$. 
As shown by the next paragraph, $\deltarw$ is the stationary distribution of $\Arw_{[v]}$. So we need to show that 

\begin{equation}
\label{eq: markov reversible}
\left[\deltarw\right]_{u}\left[\Arw_{[v]}\right]_{u,w}=\left[\deltarw\right]_{w}\left[\Arw_{[v]}\right]_{w,u}
\end{equation}

Consider two cases. 
If $\paren{u,w}\not\in\E[{[v]}]$, then $(w,u)\not\in\E[{[v]}]$ and both sides of Equation \ref{eq: markov reversible} equal 0. 
If $\paren{u,w}\in\E[{[v]}]$, then $\left[\deltarw\right]_u = \frac{\delta_{[v]}\paren{u}}{\sum_{x\in\V[{[v]}]}\delta_{[v]}\paren{x}}$ and $\left[\Arw_{[v]}\right]_{u,w}=\frac{1}{\delta_{[v]}\paren{u}}$. 
So both sides of Equation \ref{eq: markov reversible} equal $\frac{1}{\sum_{x\in\V[{[v]}]}\delta_{[v]}\paren{x}}$. 
Thus, Equation \ref{eq: markov reversible} holds and $\Arw_{[v]}$ is reversible. 

\parag{$\deltarw$ as the limiting distribution.}
By definition, the stationary distribution $\bm{\pi}$ satisfies 

\begin{equation}
\label{eq: markov stationary}
    \bm{\pi} = \bm{\pi}\Arw_{[v]}
\end{equation}

It is easy to see that setting $\bm{\pi}=\deltarw^\trans$ can satisfy Equation \ref{eq: markov stationary}. So $\deltarw$ is the stationary distribution. 
For irreducible and aperiodic Markov chain, the stationary distribution is also the \emph{limiting distribution}, and thus

\begin{equation}
    \lim_{L'\rightarrow\infty}\left[\Arw_{[v]}^{L'}\right]_{v,:} = \deltarw
\end{equation}

So far, we can already show that when the {\shadowsage} model is deep enough (\ie, with large $L'$), the model output approximates $\tau$:

\begin{equation}
    \lim_{L'\rightarrow\infty} \left[\Arw_{[v]}^{L'}\X\right]_{v,:} = \tau
\end{equation}

\parag{Desired model depth as the mixing time.}
Next, we want to see if we want the {\shadowsage} model to reach a given error bound, how many layers $L'$ are required. 
Firstly, directly applying Theorem 4.9 of \cite{mc_mixing}, we know that the error of {\shadowsage} approximating $\tau$ reduces \emph{exponentially} with the model depth $L'$. 
Then From Equation 4.36 of \cite{mc_mixing}, we can directly relate the mixing time of Markov chain with the required shaDow-SAGE depth to reach any $\epsilon$ error. 
Finally, by Theorem 12.3 of \cite{mc_mixing}, the mixing time can be bounded by the ``absolute spectral gap'' of the transition matrix $\Arw_{[v]}$. 
Note that Theorem 12.3 applies when the Markov chain is reversible -- a condition satisfied by $\Arw_{[v]}$ as we have already discussed. 
The absolute spectral gap is calculated from the eigenvalues of the transition matrix $\Arw_{[v]}$. 

In summary, {\shadowsage} can approximate $\tau$ with error decaying exponentially with depth $L'$. 
\end{proof}

\subsection{Proof on {\shadowgin} Expressivity}
\label{appendix: 1wl proof}

\begin{proof}[Proof of Theorem \ref{thm: shallow-shallow}]
We consider GNNs whose layer function is defined by Equation \ref{eq: thm shallow-shallow}. 
Define $\G[{[v]}]^L$ as the subgraph \emph{induced} from all of $v$'s $\ell$-hop neighbors in $\G$, where $1\leq \ell\leq L$.

First, we show that a {\shadow} following Equation \ref{eq: thm shallow-shallow} is at least as discriminative as the 1-dimensional Weisfeiler-Lehman test (1-WL). 
We first prove that given any $\G[{[v]}]^L$, an $L'$-layer GNN can have exactly the same output as an $L$-layer GNN, where $L'> L$. 
We note that for the target node $v$, the only difference between these two architectures is that an $L$-layer GNN exactly performs $L$ message passing iterations to propagate node information from at most $L$ hops away, while an $L'$-layer GNN has $L'-L$ more message passing iterations before performing the rest of $L$ message passings. 
Thanks to the universal approximation theorem~\cite{hornik1989multilayer}, we can let the MLPs implementing $f_1$ and $f_2$ learn the following GNN layer function: 

\begin{align*}
\vh_v^{(\ell)}=&f_1^{(\ell)}\paren{\vh_v^{(\ell-1)}, \sum_{u \in\mathcal{N}_v} f_2^{(\ell)}\paren{\vh_v^{(\ell-1)}, \vh_u^{(\ell-1)}}} \\
    =& \vh_v^{(\ell-1)}, \qquad\qquad\forall 1 \leq \ell \leq L'-L 
\end{align*}

This means $\vh_v^{(0)} = \vh_v^{(L'-L)}$. Then the $L'$-layer GNN will have the same output as the $L$-layer GNN. 
According to \cite{gin}, the normal GNN (\ie, $L$-layer on $\G[{[v]}]^L$) following Equation \ref{eq: thm shallow-shallow} is as discriminative as 1-WL. 
Thus, {\shadow} (\ie, $L'$-layer on $\G[{[v]}]^L$) following Equation \ref{eq: thm shallow-shallow} is at least as discriminative as 1-WL.

\begin{figure}
    \centering
    \begin{minipage}{0.55\textwidth}
        \centering
        \input{diagrams/example_3regular_graph.tex}
        \captionof{figure}{Example 3-regular graph and the 1-hop subgraph of the target nodes}
        \label{fig: 1wl-shadow example 2}
    \end{minipage}
    \hfill
    \begin{minipage}{0.35\textwidth}
        \centering
        \input{diagrams/example_2regular_graph.tex}
        \captionof{figure}{Example 2-regular graph with two connected components (CC)}
        \label{fig: 1wl-shadow example}
    \end{minipage}
\end{figure}

Next, we show that there exists a graph where {\shadow} can discriminate topologically different nodes, while 1-WL cannot. 
The example graph mentioned in Section \ref{sec: shadow-gin} is one such graph. We duplicate the graph here in Figure \ref{fig: 1wl-shadow example 2}. 
$\G$ is connected and is 3-regular. 
The nodes $u$ and $v$ marked in red and blue have different topological roles, and thus an ideal model / algorithm should assign different labels to them. 
Suppose all nodes in $\G$ have identical features. 
For 1-WL, it will assign the same label to all nodes (including $u$ and $v$) no matter how many iterations it runs. 
For {\shadow}, if we let {\sample} return $\G[{[v]}]^1$ and $\G[{[u]}]^1$ (\ie, $1$-hop based {\sample}), and set $L'>1$, then our model can assign different labels to $u$ and $v$. 
To see why, note that $\G[{[u]}]^1$ and $\G[{[v]}]^1$ are non-isomorphic, and more importantly, \emph{non-regular}. 
So if we run the ``\textsc{shaDow}'' version of 1-WL (\ie, 1-WL on $\G[{[v]}]^1$ and $\G[{[u]}]^1$ rather than on $\G$), the nodes $u$ and $v$ will be assigned to different labels after two iterations. 
Equivalently, {\shadow} can discriminate $u$ and $v$ with at least 2 layers. 

We can further generalize to construct more such example graphs (although such generalization is not required by the proof). 
The guideline we follow is that, the full graph $\G$ should be regular. Yet the subgraphs around topologically different nodes (\eg, $\G[{[v]}]^k$) should be non-isomorphic and non-regular. 

The graph in Figure \ref{fig: 1wl-shadow example} is another example 2-regular graph, where nodes $u$ and $v$ can only be differentiated by decoupling the GIN. 

Finally, combining the above two parts, {\shadow} following Equation \ref{eq: thm shallow-shallow} is more discriminative than the 1-dimensional Weisfeiler-Lehman test. 

\end{proof}

%% file: diagrams/deep_example_expand.tex
\definecolor{c0}{HTML}{f5f5f5}
\definecolor{c1}{HTML}{e5e5e5}
\definecolor{c2}{HTML}{d5d5d5}
\definecolor{c3}{HTML}{616161}
\definecolor{c4}{HTML}{424242}
\definecolor{r}{HTML}{FF5733}
\definecolor{g}{HTML}{0F912C}
\definecolor{bluelight}{HTML}{0F6FCA}

\tikzset{
    archnode/.style={thick,ellipse,minimum height=2cm,minimum width=3cm,c4,draw=c4,fill=c0},
}
\begin{tikzpicture}[
    outer/.style={draw=gray,dashed,fill=green!1,thick},
    >={Stealth[inset=0pt,
    length=4pt,angle'=45,round]},scale=1.,
    sty edge sel/.style={ultra thick,color=c4},
    sty node sel/.style={ultra thick,circle,c4,draw=c4,fill=c0,minimum size=\sizenode},
    sty node unsel/.style={circle,c2,draw=c2,fill=c0,minimum size=\sizenode},
    sty edge unsel/.style={color=c2},
]

\def\lena{0.5};
\def\lenb{0.5};
\def\xshift{6};

\draw (0, 0) node[circle, inner sep=1.5pt, fill=bluelight] (v) {};

\def\xa{\lena};
\def\ya{0};
\pgfmathsetmacro{\xb}{\xa + \lenb*cos(45)}%
\pgfmathsetmacro{\yb}{\ya + \lenb*sin(45)}%
\pgfmathsetmacro{\xc}{\xa + \lenb*cos(45)}%
\pgfmathsetmacro{\yc}{\ya - \lenb*sin(45)}%
\foreach \angle in {0,120,240} {
    \draw ({cos(\angle)*\xa-sin(\angle)*\ya}, {sin(\angle)*\xa+cos(\angle)*\ya}) node[circle, inner sep=1.5pt, fill=c4] (a\angle) {};
    \draw [fill=c4] (a\angle) -- (v);
    \draw ({cos(\angle)*\xb-sin(\angle)*\yb}, {sin(\angle)*\xb+cos(\angle)*\yb}) node[circle, inner sep=1.5pt, fill=c4] (b\angle) {};
    \draw [fill=c4] (b\angle) -- (a\angle);
    \draw ({cos(\angle)*\xc-sin(\angle)*\yc}, {sin(\angle)*\xc+cos(\angle)*\yc}) node[circle, inner sep=1.5pt, fill=c4] (c\angle) {};
    \draw [fill=c4] (c\angle) -- (a\angle);
}
\node[color=bluelight] at (-0.3, 0.) {$v$};
\draw (-1.3, -.4) node[circle, inner sep=1.5pt, draw=red] (other 1) {};
\node[red] at (-1.6, -.4) {$v'$};
\draw[red, dashed] (other 1) -- (c240);
\node at (-1.2, 0.25) {$\G[{[v]}]'$};

\draw ({0+\xshift}, 0) node[circle, inner sep=1.5pt, fill=bluelight] (vs) {};

\def\xa{\lena};
\def\ya{0};
\pgfmathsetmacro{\xb}{\xa + \lenb*cos(45)}%
\pgfmathsetmacro{\yb}{\ya + \lenb*sin(45)}%
\pgfmathsetmacro{\xc}{\xa + \lenb*cos(45)}%
\pgfmathsetmacro{\yc}{\ya - \lenb*sin(45)}%
\foreach \angle in {0,120,240} {
    \draw ({cos(\angle)*\xa-sin(\angle)*\ya+\xshift}, {sin(\angle)*\xa+cos(\angle)*\ya}) node[circle, inner sep=1.5pt, fill=c4] (as\angle) {};
    \draw [fill=c4] (as\angle) -- (vs);
    \draw ({cos(\angle)*\xb-sin(\angle)*\yb+\xshift}, {sin(\angle)*\xb+cos(\angle)*\yb}) node[circle, inner sep=1.5pt, fill=c4] (bs\angle) {};
    \draw [fill=c4] (bs\angle) -- (as\angle);
    \draw ({cos(\angle)*\xc-sin(\angle)*\yc+\xshift}, {sin(\angle)*\xc+cos(\angle)*\yc}) node[circle, inner sep=1.5pt, fill=c4] (cs\angle) {};
    \draw [fill=c4] (cs\angle) -- (as\angle);
}
\draw ({-1.3+\xshift}, -.4) node[circle, inner sep=1.5pt, draw=red] (other 2) {};
\draw[red, dashed] (other 2) -- (cs240);
\node[color=bluelight] at ({-0.3+\xshift}, 0.) {$v$};
\draw [green!40!black] (bs240) -- (cs0) node[midway, below right] {$e$};
\node[red] at ({-1.6+\xshift}, -.4) {$v'$};
\node at ({-1.2+\xshift}, 0.25) {$\G[{[v]}]''$};

\end{tikzpicture}

%% file: diagrams/example_2regular_graph.tex
\definecolor{colCC1}{HTML}{0F6FCA}
\definecolor{colCC2}{RGB}{255,23,68}

\tikzset{
    archnode/.style={thick,ellipse,minimum height=2cm,minimum width=3cm,c4,draw=c4,fill=c0},
}
\begin{tikzpicture}[
     outer/.style={draw=gray,dashed,fill=green!1,thick},
     >={Stealth[inset=0pt,
     length=4pt,angle'=45,round]},scale=1.,
     sty edge sel/.style={ultra thick,color=c4},
     sty node sel/.style={ultra thick,circle,c4,draw=c4,fill=c0,minimum size=\sizenode},
     sty node unsel/.style={circle,c2,draw=c2,fill=c0,minimum size=\sizenode},
     sty edge unsel/.style={color=c2},
]

\def\R{0.75};
\def\offXa{-0.6};
\def\offYa{0};
\draw ({\R*cos(0)+\offXa}, {\R*sin(0)+\offYa}) \foreach \theta in {60,120,...,360} {  -- ({\R*cos(\theta)+\offXa}, {\R*sin(\theta)+\offYa}) };
\foreach \theta in {60,120,...,360} {
     \node[inner sep=1.3pt,circle,draw=colCC1,fill=colCC1] at ({\R*cos(\theta)+\offXa}, {\R*sin(\theta)+\offYa}) {};
}
\def\offXb{1};
\def\offYb{-0.2};
\draw ({\R*cos(90)+\offXb}, {\R*sin(90)+\offYb}) \foreach \theta in {210,330,90} {  -- ({\R*cos(\theta)+\offXb}, {\R*sin(\theta)+\offYb}) };
\foreach \theta in {90,210,330} {
     \node[inner sep=1.3pt,circle,draw=colCC2,fill=colCC2] at ({\R*cos(\theta)+\offXb}, {\R*sin(\theta)+\offYb}) {};
}

\node[color=colCC1] at ({\R*cos(60)+\offXa+0.25}, {\R*sin(60)+\offYa}) {$u$};
\node[color=colCC2] at ({\R*cos(90)+\offXb+0.25}, {\R*sin(90)+\offYb}) {$v$};

\node[color=colCC2] at (\offXb, \offYb) {CC2};
\node[color=colCC1] at (\offXa, \offYa) {CC1};

\end{tikzpicture}

%% file: appendix/sampler_details.tex
\section{Designing Subgraph Extraction Algorithms}
\label{appendix: sampler detail}

We illustrate how we can follow the three main approaches (\ie, heuristic based, model based and learning based. See Section \ref{sec: sampler}) to design various subgraph extraction functions. Note that the $L$-hop and PPR based {\sample} are only example algorithms to implement the principle of ``decoupling the depth and scope''. 

\parag{Heuristic based.}
As discussed in Section \ref{sec: sampler}, we can design {\sample} by selecting appropriate graph metrics. 
We have shown some examples in Section \ref{sec: sampler}. 
Here we provide detailed description on the two {\sample} used in our experiments. 

\textsc{$L$-hop {\sample}.}\hspace{.3cm}
 Starting from the target node $v$, the algorithm traverses up to $L$ hops. At a hop-$\ell$ node $u$, the sampler
will add to its hop-($\ell+1$) node set either all neighbors of $u$, or $b$ randomly selected neighbors of $u$. The subgraph $\G[{[v]}]$ is induced from all the nodes selected by the {\sample}. 
Here depth $L$ and budget $b$ are the hyperparameters. 

\textsc{PPR {\sample}.}\hspace{.3cm}
Given a target $v$, our PPR {\sample} proceeds as follows: 
\begin{enumerate*}
\item Compute the approximate
PPR vector $\bm{\pi}_v\in\mathbb{R}^{\size{\V}}$ for $v$. 
\item Select the neighborhood $\V[{[v]}]$ such that for $u\in\V[{[v]}]$, the PPR score $\left[\bm{\pi}_v\right]_u$ is large. 
\item Construct the induced subgraph from $\V[{[v]}]$. 
\end{enumerate*}
For step 1, even though vector $\bm{\pi}_v$ is of length-$\size{\V}$, most of the PPR values are close to zero. 
So we can use a fast algorithm \cite{ppr_focs} to compute the approximate PPR score by traversing only the local region around $v$. 
Throughout the $\bm{\pi}_v$ computation, most of its entries remain as the initial value of 0. Therefore, the PPR {\sample}
is scalable and efficient w.r.t. both time and space complexity. 
For step 2, we can either select $\V[{[v]}]$ based on top-$p$ values in $\bm{\pi}_v$, or based on some threshold $\left[\bm{\pi}_v\right]>\theta$.
Then $p$, $\theta$ are hyperparameters. 
For step 3, notice that our PPR sampler only uses PPR scores as a node filtering condition. The original graph structure is still preserved among $\V[{[v]}]$,
due to the induced subgraph step. 
In comparison, other related works \cite{gdc, pprgo} do not preserve the original graph structure. 

We also empirically highlight that the approximate PPR is both scalable and efficient. The number of nodes we need to visit to obtain the approximate $\bm{\pi}_v$ is much smaller than the graph size $\size{\V}$. 
In addition, each visit only involves a scalar update, which is orders of magnitude cheaper than the cost of neighbor propagation in a GNN. We profile the three datasets, {\reddit}, {\yelp} and {\products}. 
As shown in Table \ref{tab: ppr touched neighbors}, for each target node on average, the number of nodes touched by the PPR computation is comparable to the full 2-hop neighborhood size. This also indicates that faraway neighbors do not have much topological significance. 
We further show the empirical execution time measurement on multi-core CPU machines in Figure \ref{fig: ppr time}. 

\begin{table}[!ht]
\caption{Average number of nodes touched by the approximate PPR computation}
    \centering
    \begin{tabular}{lcc}
        \toprule
        Dataset & Avg. 2-hop size & Avg. \# nodes touched by PPR\\
        \midrule\midrule
        {\reddit} & 11093 & 27751 \\
        {\yelp} & 2472 & 5575\\
        {\products} & 3961 & 5405\\
        \bottomrule
    \end{tabular}
    \label{tab: ppr touched neighbors}
\end{table}

\parag{Model based.}
From Section \ref{sec: intro}, treating $\G$ as the union of $\G[{[v]}]$'s describes a generation process on $\G$. 
{\sample} then describes the \emph{reverse} of such a process. 
This links to the graph generation literature \cite{graph_evo}. 
\eg, forest-fire generation model would correspond to {\sample} being forest-fire sampler \cite{ff_sampling}. 
More specifically, to extract the subgraph around a node $v$, we can imagine a process of adding $v$ into the partial graph $\G'$ consisting of $\V\setminus\{v\}$. 
Then the nodes selected by {\sample} would correspond to the nodes touched by such an imaginary process of adding $v$.

\parag{Learning based.}
We may treat the design of {\sample} as part of GNN training. 
However, due to the combinatorial nature of subgraph extraction, simultaneously learning {\sample} and the {\shadow} layer parameters is challenging. 
The learning of {\sample} can be made possible with appropriate approximation and relaxations. \eg, we may use the techniques proposed in \cite{gnnexplainer} to design a two-phased learning process. 
In Phase 1, we train a normal $L$-layer GNN, and then use \cite{gnnexplainer} to identify the important neighbors among the full $L$-hop neighborhood. 
In Phase 2, we use the neighborhood identified in Phase 1 as the subgraph returned by {\sample}. Then we train an $L'$-layer {\shadow} on top of such neighborhood. 

Detailed design and evaluation on the model and learning based approaches are left as future work.

%% file: appendix/complete_framework.tex
\section{The Complete {\shadow} Framework}
\label{appendix: pool defn}

\begin{algorithm}[tb]
    \caption{Inference algorithm for the general {\shadow} model}
    \label{algo: shadow complete}
    \begin{algorithmic}
    \State {\bfseries Input:} $\G\paren{\V, \E, \X}$; Target nodes $\V[t]$; GNN model; $C$ number of samplers $\{{\sample}_i\}$;
    \State {\bfseries Output:} Node embedding matrix $\bm{Y}_t$ for $\V[t]$;
      \For{$v\in\V[t]$}
        \For{$i=1$ to $C$}
            \State Get $\G[{[v], i}]$ by $\sample_i$ on $\G$
            \State Propagate $\G[{[v], i}]$ in the $i$\textsuperscript{th} branch of $L$-layer GNN
            \State $\bm{H}_{[v], i} \gets f_i^L\paren{\X_{[v], i}, \mA_{[v],i}}$
            \State $\bm{y}_{v,i}\gets \func[MLP]{\func[READOUT]{\bm{H}_{[v], i}}\big\| \left[\bm{H}_{[v], i}\right]_{v,:}}$
        \EndFor
        \State $\bm{y}_v\gets\func[ENSEMBLE]{\{\bm{y}_{v, i}\}}$
      \EndFor
    \end{algorithmic}
\end{algorithm}

Algorithm \ref{algo: shadow complete} shows the inference algorithm of {\shadow} after integrating the various architecture extensions discussed in Section \ref{sec: ensemble}. 
The $f_i$ function specifies the layer propagation of a given GNN architecture (\eg, GCN, GraphSAGE, \emph{etc.}), and $f_i^L$ is a shorthand for $L$ times iterative layer propagation of the $L$-layer model.  
The $\func[READOUT]{\cdot}$ function performs subgraph pooling and $\func[ENSEMBLE]{\cdot}$ performs subgraph ensemble. 
The implementation of such functions are described in the following subsections. 

After the pooling by $\func[READOUT]{\cdot}$, we further feed into an MLP the vector summarized from the subgraph and the vector for the target. 
This way, even if two target nodes $u$ and $v$ have the same neighborhood, we can still differentiate their embeddings based on the vectors $\left[\bm{H}_{[v]}\right]_{v,:}$ and $\left[\bm{H}_{[u]}\right]_{u,:}$.

\subsection{Architecture for Subgraph Pooling}

The $\func[READOUT]{\cdot}$ function in Algorithm \ref{algo: shadow complete} can perform subgraph-level operation such as sum pooling, max pooling, mean pooling, sort pooling \citep{sortpool}, \emph{etc.}
Table \ref{tab: pool ops} summarizes all the pooling operations that we have integrated into the {\shadow} framework, where $\bm{H}_{[v]}$ is the subgraph embedding matrix as shown in Algorithm \ref{algo: shadow complete}.

\begin{table}[!ht]
    \caption{$\func[READOUT]{\cdot}$ function for different pooling operations.}
        \centering
        \begin{tabular}{lc}
            \toprule
            Name & $\func[READOUT]{\cdot}$\\
            \midrule\midrule
            Center & $\left[\bm{H}_{[v]}\right]_{v,:}$\\
            Sum & $\sum_{u\in\V[{[v]}]}\left[\bm{H}_{[v]}\right]_{u, :}$\\
            Mean & $\frac{1}{\size{\V[{[v]}]}}\sum_{u\in\V[{[v]}]}\left[\bm{H}_{[v]}\right]_{u, :}$\\
            Max & $\max_{u\in\V[{[v]}]} \left[\bm{H}_{[v]}\right]_{u, :}$\\
            Sort & $\func[MLP]{\left[\bm{H}_{[v]}\right]_{\left[\arg\text{sort} \left[\bm{H}_{[v]}\right]_{:,-1}\right]_{:s}}}$\\
            \bottomrule
        \end{tabular}
        \label{tab: pool ops}
    \end{table}

In particular, sort pooling 
\begin{enumerate*}
\item sorts the last column of the feature matrix,
\item takes the indices of the top $s$ values after sorting ($s$ is a hyperparameter of the sort pooling function), and
\item slice a submatrix of $\bm{H}_{[v]}$ based on the top-$s$ indices.
\end{enumerate*}
The input to the $\func[MLP]{\cdot}$ (last row of Table \ref{tab: pool ops}) is a submatrix of $\bm{H}_{[v]}$ consisting of $s$ rows, and the output of the $\func[MLP]{\cdot}$ is a single vector.

\subsection{Architecture for Subgraph Ensemble}
\label{appendix: arch ensemble}

For the ensemble function, we implement the following after applying attention on the outputs of the different model branches:

\begin{align}
    w_i =& \func[MLP]{\bm{y}_{v,i}} \cdot \bm{q}\\
    \bm{y}_v =& \sum_{i=1}^{C}\widetilde{w}_i\cdot \bm{y}_{v,i}
\end{align}

where $\bm{q}$ is a learnable vector; $\widetilde{\bm{w}}$ is normalized from $\bm{w}$ by softmax; $\bm{y}_v$ is the final embedding.

%% file: plots/gpu_mem_speed.tex
\definecolor{colA}{RGB}{255,23,68}  
\definecolor{colB}{HTML}{B6A807}    
\definecolor{colC}{HTML}{4AAC2B}    
\definecolor{colD}{HTML}{0F6FCA}    
\definecolor{skyblue}{RGB}{213,0,0}
\definecolor{yellowgreen}{RGB}{255,23,68}
\definecolor{red1}{HTML}{FB7E78}
\definecolor{c3}{HTML}{0F6FCA}

\def\plotwidth{0.6\linewidth}
\def\plotheight{0.4\linewidth}
\def\maxmem{12}
\def\minspeed{0.}
\def\maxspeed{2.5}
\def\plotW{0.59}

\begin{tikzpicture}
\begin{axis}[
    height=\plotheight,
    width=\plotwidth,
    grid,
    label style = {font = {\fontsize{9.5 pt}{12 pt}\selectfont}},
    tick label style = {font = {\fontsize{8.5 pt}{12 pt}\selectfont}},
    ymin=\minspeed, ymax=\maxspeed,
    xmin=0, xmax=\maxmem,
    ylabel=Execution speedup ($\times$),
    xlabel=GPU (RTX 3090) memory consumption (GB),
    legend cell align=left,
    legend columns=2,
    legend style={
        draw=none,
        fill=none,
        nodes={scale=0.75,},
        at={(1.05, 1.2)}},
]
\addplot[colA, mark=*] coordinates{
    (2.3, 1)
    (4.2, 1.257)
    (9.5, 1.1364)
};
\addplot[colB, mark=*] coordinates{
    (2.35, 1)
    (4.6, 1.71)
    (10, 1.4)
};

\legend{{\products}, {\papers}}
\end{axis}
\end{tikzpicture}

%% file: data/table2_gin_jk.tex
\begin{tabular}{lcccccc}
    \toprule
    & \multicolumn{2}{c}{\flickr} & \multicolumn{2}{c}{\reddit} & \multicolumn{2}{c}{\arxiv} \\
    \cmidrule(lr){2-3}\cmidrule(lr){4-5}\cmidrule(lr){6-7}& Normal & {\textsc{shaDow}} & Normal & {\textsc{shaDow}} & Normal & {\textsc{shaDow}}\\
    \midrule
    \midrule
    JK (3) & 0.4945\std{0.0070} & \textbf{0.5317}\std{0.0027} & 0.9649\std{0.0010} & \textbf{0.9682}\std{0.0003} & 0.7130\std{0.0026} & \textbf{0.7201}\std{0.0017}\\
    JK (5) & 0.4940\std{0.0083} & \textbf{0.5328}\std{0.0026} & 0.9640\std{0.0013} & \textbf{0.9685}\std{0.0006} & 0.7166\std{0.0053} & \textbf{0.7226}\std{0.0024}\\
    \cmidrule(lr){2-7}
    GIN (3) & 0.5132\std{0.0031} & \textbf{0.5228}\std{0.0028} & 0.9345\std{0.0034} & \textbf{0.9578}\std{0.0006} & 0.7087\std{0.0016} & \textbf{0.7173}\std{0.0029}\\
    GIN (5) & 0.5004\std{0.0067} & \textbf{0.5255}\std{0.0023} & 0.7550\std{0.0039} & \textbf{0.9552}\std{0.0007} & 0.6937\std{0.0062} & \textbf{0.7140}\std{0.0027}\\
    \bottomrule
\end{tabular}

%% file: plots/ppr_time.tex
\definecolor{colA}{RGB}{255,23,68}  
\definecolor{colB}{HTML}{B6A807}    
\definecolor{colC}{HTML}{4AAC2B}    
\definecolor{colD}{HTML}{0F6FCA}    
\definecolor{skyblue}{RGB}{213,0,0}
\definecolor{yellowgreen}{RGB}{255,23,68}
\definecolor{red1}{HTML}{FB7E78}
\definecolor{c3}{HTML}{0F6FCA}

\def\plotwidth{0.6\textwidth}
\def\plotheight{0.4\textwidth}

\begin{tikzpicture}
\begin{axis}[
    height=\plotheight,
    width=\plotwidth,
    tick label style={font=\footnotesize},
    title style={font=\small,at={(axis description cs:0.5, 0.95)}},
    grid,
    label style = {font = {\fontsize{9.5 pt}{12 pt}\selectfont}},
    tick label style = {font = {\fontsize{8.5 pt}{12 pt}\selectfont}},
    enlarge x limits=0.2,ymin=0,ymax=3,
    ybar=\pgflinewidth,
    tick align=inside,
    xtick=data,
    x tick label style={rotate=45},
    bar width=4pt,
    ymajorgrids = true,
    tick label style={font=\small},
    symbolic x coords={{\flickr}, {\reddit}, {\yelp}, {\arxivshort}, {\productsshort}, {\papersshort}},
    ylabel=Inference time per node (ms),
    width=\plotwidth,
    height=\plotheight,
    y label style={at={(axis description cs:-0.1,.37)},anchor=south},
    yticklabel style={/pgf/number format/precision=2,/pgf/number format/fixed},
    legend style={font=\small},
    legend cell align=left,
    legend columns=3,
    legend style={at={(0.52,1.27)},anchor=north,/tikz/column 5/.style={column sep=5pt},draw=none,/tikz/every even column/.append style={column sep=0.1cm}},
    label style = {font = {\fontsize{9.5 pt}{12 pt}\selectfont}},
]
\addplot[style={c3,fill=c3,mark=none}]
            coordinates {
            ({\flickr},0.9) 
            ({\reddit},0.33) 
            ({\yelp},0.2) 
            ({\arxivshort},0.82) 
            ({\productsshort}, 0.52)
            ({\papersshort}, 0.63)};
\addplot[style={red1,fill=red1,mark=none}]
            coordinates {
            ({\flickr}, 0.66)
            ({\reddit}, 1.09) 
            ({\yelp},.35) 
            ({\arxivshort}, 0.68)
            ({\productsshort},0.9)
            ({\papersshort}, 1.32)};  
\addplot[style={yellowgreen,fill=yellowgreen,mark=none}]
            coordinates {
            ({\flickr}, 1.98) 
            ({\reddit},1.73) 
            ({\yelp},1) 
            ({\arxivshort},1.18)
            ({\productsshort},2.5)
            ({\papersshort}, 2.23)};

\legend{PPR,{\shadowsage},{\shadowgat}}

\end{axis}
\end{tikzpicture}

%% file: plots/acc_time_tradeoff.tex
\tikzset{mark size=3}
\definecolor{red1}{HTML}{FB7E78}
\definecolor{red2}{RGB}{255,23,68}
\definecolor{blue1}{HTML}{78C5FB}
\definecolor{blue2}{HTML}{0F6FCA}

\begin{tikzpicture}
\def \shadowSAGE{red1}
\def \shadowGAT{red2}
\def \SAGE{blue1}
\def \GAT{blue2}

\def \dataF{+}
\def \dataR{x}
\def \dataY{star}
\def \dataA{Mercedes star}
\def \dataP{asterisk}

\def \op{0.2}
\begin{axis}[
    xmode=log,
    ymin=0.4,ymax=1,xmin=0.1,xmax=1000,
    xmajorgrids=true,
    ymajorgrids=true,
    width=0.6\linewidth,
    height=0.4\linewidth,
    ylabel=Test accuracy,
    xlabel=Inference time per node (ms),
    title=Tradeoff analysis (5-layer),
    title style={at={(0.47,0.93)}},
    scatter/classes={%
    shadowsage={mark=square*,\shadowSAGE,mark size=4pt},
        shadowgat={mark=square*,\shadowGAT,mark size=4pt},
        sage={mark=square*,\SAGE,mark size=4pt},
        gat={mark=square*,\GAT,mark size=4pt},
        flickr={mark=\dataF,black!60,thick},
        reddit={mark=\dataR,black!60,thick},
        arxiv={mark=\dataA,black!60,thick},
        products={mark=\dataP,black!60,thick},
        shadowsageflickr={mark=\dataF,\shadowSAGE},
        sageflickr={mark=\dataF,\SAGE},
        shadowgatflickr={mark=\dataF,\shadowGAT},
        gatflickr={mark=\dataF,\GAT},
        shadowsagereddit={mark=\dataR,\shadowSAGE},
        sagereddit={mark=\dataR,\SAGE},
        shadowgatreddit={mark=\dataR,\shadowGAT},
        gatreddit={mark=\dataR,\GAT},
        shadowsagearxiv={mark=\dataA,\shadowSAGE},
        sagearxiv={mark=\dataA,\SAGE},
        shadowgatarxiv={mark=\dataA,\shadowGAT},
        gatarxiv={mark=\dataA,\GAT},
        shadowsageproducts={mark=\dataP,\shadowSAGE},
        sageproducts={mark=\dataP,\SAGE},
        shadowgatproducts={mark=\dataP,\shadowGAT},
        gatproducts={mark=\dataP,\GAT}
    },
    legend style={nodes={scale=0.8, transform shape},cells={align=left}},
    legend cell align=left,
    legend columns=1,
    legend style={at={(1.3,1.)},anchor=north,/tikz/column 4/.style={column sep=5pt},draw=none,/tikz/every even column/.append style={column sep=0.07cm}},
    y label style={at={(axis description cs:-0.15,.45)},anchor=south},
    label style = {font = {\fontsize{9.5 pt}{12 pt}\selectfont}},
    tick label style = {font = {\fontsize{8.5 pt}{12 pt}\selectfont}},
]
\addplot[scatter,thick,only marks,%
    scatter src=explicit symbolic] 
    table[meta=label] {
    x  y      label
    1.194	0.4875	shadowgatflickr
    1.451	0.5132	shadowgatflickr
    1.719	0.5281	shadowgatflickr
    1.983	0.538	shadowgatflickr
    1.111	0.4949	shadowsageflickr
    1.307	0.5247	shadowsageflickr
    1.512	0.5341	shadowsageflickr
    1.704	0.5407	shadowsageflickr
};
    
\addplot[scatter,thick,only marks,%
    scatter src=explicit symbolic] 
    table[meta=label] {
    x  y      label
    462  0.480  sageflickr
    500  0.522  gatflickr
    80   0.719  sagearxiv
    175  0.720  gatarxiv
};

\addplot[scatter,thick,only marks,%
    scatter src=explicit symbolic] 
    table[meta=label] {
    x  y      label
    0.808	0.9611	shadowgatreddit
    1.436	0.9693	shadowgatreddit
    2.134	0.9723	shadowgatreddit
    2.801	0.9727	shadowgatreddit
    0.61	0.9572	shadowsagereddit
    0.973	0.9661	shadowsagereddit
    1.458	0.9697	shadowsagereddit
    1.985   0.9695	shadowsagereddit
};

\addplot[scatter,thick,only marks,%
    scatter src=explicit symbolic] 
    table[meta=label] {
    x  y      label
    1.135	0.7227	shadowgatarxiv
    1.516	0.73	shadowgatarxiv
    1.894	0.7318	shadowgatarxiv
    2.28	0.7321	shadowgatarxiv
    1.117	0.7206	shadowsagearxiv
    1.433	0.7248	shadowsagearxiv
    1.707	0.7256	shadowsagearxiv
    1.943	0.7263	shadowsagearxiv
};

\addplot[scatter,thick,only marks,%
    scatter src=explicit symbolic] 
    table[meta=label] {
    x  y      label
    0.97	0.7932	shadowgatproducts
    1.64	0.8028	shadowgatproducts
    2.20	0.806	shadowgatproducts
    2.61	0.8074	shadowgatproducts
    0.88	0.7874	shadowsageproducts
    1.378	0.797	shadowsageproducts
    1.809	0.8001	shadowsageproducts
    2.08	0.8015	shadowsageproducts
};

\legend{{\textsc{shaDow}\\SAGE},{\textsc{shaDow}\\GAT},{Normal\\SAGE},{Normal\\GAT},{\flickr},{\reddit},{\arxivshort},{\productsshort}}
\end{axis}
\end{tikzpicture}

%% file: data/table3_ensemble.tex
\begin{tabular}{clcc}
    \toprule
    $L'$ & \sample & {\flickr}  & {\products} \\
    \midrule
    \midrule
     3 & PPR & 0.5257\std{0.0021} &  0.7773\std{0.0032}\\
    \cmidrule(lr){2-4}
    \multirow{3}{*}{5} &$2$-hop & 0.5210\std{0.0023} & 0.7794\std{0.0039}\\
    & PPR & 0.5273{\std{0.0020}}  & 0.7836\std{0.0034} \\
    & Ensemble & \textbf{0.5304}\std{0.0017} & \textbf{0.7858}\std{0.0021} \\
    \cmidrule(lr){2-4}
    7 & PPR & 0.5225\std{0.0023} & \textbf{0.7850}\std{0.0044}\\
    \bottomrule
\end{tabular}